\documentclass{article}

\usepackage{microtype}
\usepackage{graphicx}
\usepackage{subcaption}
\usepackage{tabularx}
\usepackage{booktabs} % for professional tables

\usepackage{hyperref}
\usepackage[accepted]{icml2026}

\usepackage{mathrsfs}
\usepackage{bm}
\usepackage{amsmath}
\usepackage{amssymb}
\usepackage{mathtools}
\usepackage{amsthm}
\usepackage{comment}
\usepackage{setspace}

\usepackage[capitalize,noabbrev]{cleveref}

\theoremstyle{plain}
\newtheorem{theorem}{Theorem}

\newtheorem{lemma}{Lemma}
\newtheorem{corollary}[theorem]{Corollary}
\theoremstyle{definition}
\newtheorem{definition}{Definition}
\newtheorem{assumption}{Assumption}
\theoremstyle{remark}
\newtheorem{remark}{Remark}
\usepackage{xcolor}

\usepackage[textsize=tiny]{todonotes}

\icmltitlerunning{Rationality Measurement and Theory for Reinforcement Learning Agents} %define a rational action as one that max-imises the hidden true value functionearning Agents}

\begin{document}

\twocolumn[
  \icmltitle{Rationality Measurement and Theory for Reinforcement Learning Agents}



  \icmlsetsymbol{equal}{*}

  \begin{icmlauthorlist}
    \icmlauthor{Kejiang Qian}{ed}
    \icmlauthor{Amos Storkey}{ed}
    \icmlauthor{Fengxiang He}{ed}
  \end{icmlauthorlist}

  \icmlaffiliation{ed}{University of Edinburgh}
%  \icmlaffiliation{comp}{Company Name, Location, Country}
%  \icmlaffiliation{sch}{School of ZZZ, Institute of WWW, Location, Country}

  \icmlcorrespondingauthor{Fengxiang He}{fhe@ed.ac.uk}

  \icmlkeywords{Machine Learning, ICML}

  \vskip 0.3in
]


% This command actually createse define a rational action as one that max-
% imises the hidden true value function the footnote in the first column listing the

\printAffiliationsAndNotice{}  % no special notice (required even if empty)

\begin{abstract}
This paper proposes a suite of rationality measures and associated theory for reinforcement learning agents, a property increasingly critical yet rarely explored. We define an action in deployment to be perfectly rational if it maximises the hidden true value function in the steepest direction. The expected value discrepancy of a policy's actions against their rational counterparts, culminating over the trajectory in deployment, is defined to be expected rational risk; an empirical average version in training is also defined. Their difference, termed as rational risk gap, is decomposed into (1) an extrinsic component caused by environment shifts between training and deployment, and (2) an intrinsic one due to the algorithm's generalisability in a dynamic environment. They are upper bounded by, respectively, (1) the $1$-Wasserstein distance between transition kernels and initial state distributions in training and deployment, and (2) the empirical Rademacher complexity of the value function class. Our theory suggests hypotheses on the benefits from regularisers (including layer normalisation, $\ell_2$ regularisation, and weight normalisation) and domain randomisation, as well as the harm from environment shifts. Experiments are in full agreement with these hypotheses. The code is available at \href{https://github.com/EVIEHub/Rationality}{https://github.com/EVIEHub/Rationality}.
\end{abstract}

\section{Introduction}
Reinforcement learning is rapidly advancing toward human-level capabilities in many domains, such as robotics \citep{nguyen2019review}, autonomous vehicles \citep{feng_dense_2023}, finance \citep{liu2022nips}, and reasoning in large language models (LLMs) {\citep{shao2024deepseekmath}}. They are increasingly embedded in real-world, high-stakes systems that directly impact human lives and social fabric. For example, we can expect to share public roads with autonomous vehicles in the near future; in financial markets, reinforcement learning already accounts for a substantial proportion of trading activities. The increasing penetration of reinforcement learning agents in society calls for an understanding of their behaviours through the economic lens. Rationality is fundamental to this end: it characterises {agent behaviour} in decision making that maximise their utilities given accessible information, making it possible to economically study agent behaviours \citep{vonneumannmorgan2004, dayan2008decision,sen1994formulation}.

%It is also a key issue to consider regarding robust behaviour of agents - how much can we be sure agents behave sensibly, given the premise we provide?

%However, real agents do not always meet idealised assumptions like unlimited computation or perfect information. This motivates \emph{bounded rationality}, which studies structured deviations from optimal choice arising from cognitive limits, computational costs, limited attention, or misspecified beliefs \citep{simon1990bounded, conlisk1996bounded}.

% the rationality of reinforcement learning is rarely touched on in the literature.
%To address the issues, this paper proposes a suite of rationality measures for reinforcement learning agents and develops theory based on the measures. 

We mathematically define an action to be {\it perfectly rational} if it maximises the actual value function (though it might be unknown) in the steepest direction. {An agent can not be perfectly rational, i.e., of bounded rationality \citep{simon1990bounded, conlisk1996bounded}, leading to loss in {action-}value function, defined as {\it rational value loss}. This paper is particularly interested in the rationality in deployment (or ``inference''). {Cumulating the expected rational value loss over the trajectory in inference, we define an {\it expected rational value risk}. This measure is not directly accessible; we then define an estimator, {\it empirical rational value risk}, to be the empirical average version in training.} Their difference, termed {\it rational risk gap}, measures the rationality of agents in deployment, given their observable behaviour in training, which is central in the theoretical development in this paper.} %We also define the asymptotic rational properties of an agent, such as whether the agent is asymptotically rational. 
To note, this suite of measures takes a ``local and immediate'' perspective: an action is defined to be rational if this {\it individual} move is optimal, given all information {\it available at that time}. %We do not expect a rational agent to be able to overview the global landscape or anticipate the future. 
This setting coincides with a large volume of literature in economics, such as \citet{sen2002rationality,samuel2015ration}.   

{
The rational risk gap is decomposed into two components: 
(1) an {\it extrinsic rational gap}, caused by the environment shifts between training and deployment,
and (2) an {\it intrinsic rational gap}, determined by the algorithm itself. %\bry{quantifying the difference between the expected value under the training environment and its empirical value, affected by both value function class complexity and policy updates.} 
This decomposition provides a lens for understanding the sources of sub-rationality.}
%To demonstrate the applicability of our framework, we apply it to a generalised reinforcement learning problem, 
%We then theoretically analyse the rationality of an agent deployed in an inference environment but trained in a different source environment. 
We prove that the {two components} are upper bounded as follows. %of her policy $\pi$ has an upper bound composed of three components: a \emph{sequential prediction error}, an \emph{environment shift}, and an \emph{intrinsic rational gap}.
% The \textbf{sequential prediction error $\text{SPE}$} is bounded by the sequential Rademacher complexity $\mathfrak R_t(\mathcal V)$ of the value-function class $\mathcal V$ along training.
% %We first establish a high-probability upper bound on \fh{{sequential prediction error}, $\text{SPE}$}. 
% For any positive constants $c_1,c_2 >0$, %we prove that 
% with probability at least $1-\delta$, we have  
% \begin{align*}
% \text{SPE}&\le
%     c_1\left[\beta \cdot\mathfrak{R}_t(\mathcal{V})+
%     \frac{H}{t}\,\log\frac{1}{\delta}\right],
% \end{align*}
% where $\beta=1 + \sqrt{\log(1/\delta)}\,\log^{3/2}(eT)$. %, and $\mathfrak R_t(\mathcal V)$ denotes the sequential Rademacher complexity of the value-function class. 
% \fh{This result demonstrates the effect of \emph{distribution-dependent sampling} in reinforcement learning.}
%We then upper bound t
The {extrinsic rational gap} is bounded by 
%\begin{align*}
$L_s H\cdot W_1(p_0^\dagger,p_0)+H^2L_s(L_p+1)\cdot W_1(p^\dagger, p)$,
%\end{align*}
relying on the {$1$-Wasserstein distance $W_1(p^\dagger_0,p_0)$ between initial state distributions $p^\dagger_0$ in inference and $p_0$ in training, $1$-Wasserstein distance $W_1(p^\dagger,p)$ between transition kernels $p^\dagger$ in inference and $p$ in training, Lipschitz constant $L_s$ of {the mapping from state to value function}, Lipschitz constant $L_p$ of the mapping from transition kernel to its induced state distributions, and horizon $H$ of an episode. This term may help understand the \emph{sim-to-real transfer} challenge \citep{da2025survey}}. %W_1(p^\dagger,p)$ is the $1$-Wasserstein distance between the inference and training transition kernels, and 
The {intrinsic rational gap} has an upper bound 
%\begin{align*}
    $L_\Pi H\sqrt{2\log|\mathcal A|}+2\sum_{h=1}^H\hat{\mathfrak{R}}_h(\mathcal{Q}_\Pi)+3H^2\sqrt{\frac{\log(4H/\delta)}{2T}}$,
%\end{align*}
for any $\delta \in (0,1)$, with probability at least $1-\delta$, relying on the empirical Rademacher complexity of value-function class $\mathcal Q_\Pi$, %$\hat{\mathfrak R}(\mathcal V)$ denotes the empirical Rademacher complexity of the value-function class $\mathcal V$, 
Lipschitz constant $L_\Pi$ of the mapping from policy $\pi$ to its induced state distribution, %$H$ is the horizon length,   
the action space cardinality $|\mathcal A|$, and {the training episode number $T$.}

{
Our theory suggests empirically testable hypotheses: (1) regularisers, {including layer normalisation \cite{ba2016layer}, $\ell_2$-regularisation, and {weight normalisation \cite{salimans2016weight}} control the hypothesis complexity of value function class, contributing positively to rationality; (2) domain randomisation \cite{tobin2017domainrandomization} {improves robustness across environments}, also making benefits to rationality, and (3) environment shifts between training and deployments, are harmful to rationality. We conduct experiments to verify these hypotheses, employing Deep Q-Network (DQN) \citep{mnih2013playing} on the Taxi-v3 \citep{Dietterich2000} and Cliff Walking environments \citep{sutton18}. %under a range of regularisation techniques, including batch normalisation, $\ell_2$-regularisation, network randomisation, and randomised convolution, as well as data augmentation strategies such as random cropping and cutout, and varying numbers of training levels. 
The empirical results are in full agreement with the hypotheses. %The code has been submitted and will be publicly available. %that the rational risk gap can be effectively reduced through regularisation and data augmentation, while environment shift between training and inference environments has negative impacts on rationality, in full agreement with our theoretical results.}
}
%The code is available at \href{https://github.com/EVIEHub/Rationality}{https://github.com/EVIEHub/Rationality}.

To our best knowledge, this work is the first to develop a mathematical framework for measuring the rationality of reinforcement learning agents. %indicating sources of irrationality.
%\bry{It attributes the irrationality to imperfect generalisation and environment shifts, on top of value misalignment in training.} 
Our theory sheds light on understanding and improving the rationality of reinforcement learning, which is increasingly critical in this era, as we are inevitably and irreversibly marching into a human-AI co-existing society. %formally upper bounds the sources of sub-rationality of agents and theoretically suggests practical implications for generalised reinforcement learning problems.

% \paragraph{Conflict of Interest Disclosure}
% The authors declare no financial conflicts of interest that could reasonably be perceived to influence this work.

\subsection{Related Works}

%\paragraph{Rationality in game theory} Rationality is in the fundamental place of game theory and economics; usually, a rational agent is defined to make decisions that maximise her utility under certain constraints \citep{vonneumannmorgan2004, dayan2008decision}. %The individual and societal behaviour of agents with 
%`Bounded rationality' (i.e., imperfect rationality) is also extensively studied \citep{simon1990bounded, conlisk1996bounded}. 

%\paragraph{Rationality in game theory}
%Rationality is central to game theory and economics. 
%In classical models, a 
%A rational agent chooses actions that maximise her (expected) utility subject to available information and constraints \citep{vonneumannmorgan2004, dayan2008decision}. This serves as a normative benchmark and underlies standard solution concepts such as best responses and equilibrium behaviour.

\paragraph{Rationality of Machine Learning} 
Efforts to study the rationality of machine learning are seen in the literature. \citet{valiant1995rationality} proposes a philosophical definition: rationality is the ability to abstract and utilise available information to understand, predict, and control the environment, with a probably approximately correct (PAC) style criterion. {\citet{abel2019concepts} provides a formal characterisation for bounded rationality of reinforcement learning, showing that rational decisions depend on how agents represent environments, balancing simplicity and predictive accuracy. Analysing behavioural data from human participants, \citet{evans2025rational} introduces the Wasserstein distance between the learned policy and prior as a constraint to model bounded rationality in reinforcement learning. {\citet{sunehag15} establish decision-theoretic axioms of rational reinforcement learning agents, but these exclude a large class of commonly used algorithms, such as those relying on $\epsilon$-greedy exploration, which are evaluated as irrational and out of their scope.} Despite these conceptual formalisations and empirical works, a theoretical framework remains absent, summarised by \citet{macmillanscott2025}.

\paragraph{Value Alignment in Reinforcement Learning} 
The reinforcement learning literature has seen extensive efforts on aligning agents' value with some optimal value, either explicitly or implicitly. For example, the “classic” reinforcement learning is usually around optimising regret, defined as the cumulative suboptimality in terms of getting rewards \cite{jin2020provably,azar2017regret}. The same applies in using reinforcement learning for LLM value alignment \cite{shao2024deepseekmath}. However, even if an agent (or LLM) has perfectly learned the optimal value, it can still behave suboptimally in deployment (or reasoning), because the agent fails to take the actions that optimise the aligned value.

\paragraph{Generalisation in Reinforcement Learning}
Generalisation in reinforcement learning is more subtle than in supervised learning because here data is generated by correlated trajectories, and learned policies influence observation. Existing papers establish within-environment guarantees in finite MDPs via PAC and regret analyses \citep{strehl2009pac, jaksch2010ucrl2, azar2017regret}; and approximate dynamic programming characterises how estimation and approximation errors propagate through Bellman backups \citep{munos2008fitted}. More recent work replaces dependence on state space with structural complexity measures for rich observations, e.g., Eluder Dimension and Bellman-type ranks \citep{russo2013eluder, jiang2017bellmanrank, sun2019witnessrank, chijin21}. For deep reinforcement learning, \citet{liu2022understanding} casts temporal-difference error as a generalisation problem under neural function approximation; and \citet{wang19generalisation} analyses the generalisation gap in the reparameterisable settings. %, decomposing intrinsic and extrinsic sources via the reparameterisation trick; the latter views extrinsic error as a distribution shift between training and test environments. %Offline RL further sharpens this issue because fixed datasets induce coverage constraints and distribution shift, motivating pessimistic/conservative learning principles \citep{kidambi2020morel, rashidinejad2021pessimistic}, while OOD evaluation across procedurally generated environments remains an important empirical stress test \citep{cobbe2020procgen}

\section{Preliminaries}

\paragraph{Episodic Markov Decision Process (EMDP)} 
Suppose an agent, at state $s \in \mathcal{S}$, takes an action $a$ from a finite space $\mathcal{A}$ that transits her to state $s'$ sampled from transition kernel $p(\cdot \mid s, a) \in \Delta(\mathcal{S})$, and then receives an immediate reward $r$. %Fix a function class $\Pi$, We use 
The action is sampled from policy $\pi \in \Pi,\;\pi: \mathcal S \to \Delta(\mathcal A)$, relying on %to represent the probability over the taken actions given a 
state $s$.
We assume the learning process is ``episodic'': agents, in every episode, start at initial states $s_1$ drawn from %$ sampled from \bry{initial state 
distribution $p_0(\cdot) \in \Delta(\mathcal{S})$, run for $H$ time steps (i.e., the horizon), and yield returns $\sum_{h=1}^H r_h$ ($r_h \in [0,1]$). %after this episode, the agent returns to the initial state $s_1$ and iterates the process. 
Intermediate policies $\pi_t=\{\pi^t_h\}_{h=1}^H$ are generated during the training of {$T$ episodes} in total. {Given a policy $\pi$ and a transition kernel $p$, {a trajectory \(\mathbf{s}^{t}_{h:H}=(s_{h}^{t},s_{h+1}^{t},\cdots,s_{H}^{t})\) is taken in episode $t$.} %We illustrate our formulation using 
This setting is termed EMDP \( \mathcal{M} = (\mathcal{S}, \mathcal{A}, H,  \{r_h\}_{h=1}^H, \{p_h\}_{h=1}^H,p_0)\).\

% \fh{We work on a probability space $(\Omega,\mathscr F,\mu)$ endowed with a filtration \(
% \mathscr F_0 \subset \mathscr F_1 \subset \cdots\) and $\mathscr F_0 = \{\emptyset,\Omega\}$. At the beginning of episode $t$ and time step $h$, the algorithm selects a policy $\pi_t$ that is $\mathscr F_{t-1}$-measurable, and 
% \fh{Training and inference}
%\paragraph{Value functions} 
%The expected {return} 
A policy can be evaluated by action-value function \(Q_h^{\pi}(s,a)=\mathbb E_\pi \left[\sum_{j=h}^H r_j \mid s_h=s,a_h=a\right]\) 
and value function \(V_h^{\pi}(s)=\mathbb E_\pi \left[ \sum_{j=h}^H r_j \mid s_h=s \right]\). 
%\begin{align*}
%    Q_h^{\pi}(s,a)=&\mathbb E_\pi \left[\sum_{j=h}^H r_j \mid s_h=s,a_h=a\right],\\
%    V_h^{\pi}(s)=&\mathbb E_\pi \left[ \sum_{j=h}^H r_j \mid s_h=s \right]
%\end{align*}%that an agent takes \fh{an action $a_h$ following the policy \(\pi\) at the state $s$ and the next state $s_{h+1}$} following the transition kernel $p$ can obtain as below,
A terminal condition $V_{H+1}^{\pi}(s) = 0$ indicates that no reward is gained beyond the horizon $H$. Value functions are recursively defined,
%By the definitions above, we 
governed by the Bellman equations:
\(
V_h^{\pi}(s)
=
\mathbb{E}_{a\sim\pi(\cdot|s)}
\!\left[
Q_h^{\pi}(s,a)
\right]
\), 
\(
Q_h^{\pi}(s,a)
=
r_h
+
\mathbb{E}_{s' \sim p_h(\cdot|s,a)}
\!\left[
V_{h+1}^{\pi}(s')
\right]
\).

\paragraph{Training-to-Deployment Shifts}
Suppose the training environment has transition kernels $p=\{p_h\}_{h=1}^H$ and an initial state distribution $p_0$. A policy $\pi$ induces a state distribution $\mathcal D_h^{\pi}$ at time step $h$, termed \emph{state distribution in training} {\citep{cobbe2020a,wang2020improving}}.
Similarly, the deployment environment has %from the training one.
different transition kernels $p^\dagger=\{p_h^\dagger\}_{h=1}^H$ and initial distribution $p_0^\dagger$, %, resulting in an \emph{inference distribution}. 
under which a \emph{state distribution in deployment} $\mathcal D_h^{\pi,\dagger}$ is induced. %{\citep{cobbe2020a,wang2020improving}}. 

Following \citet{liu2022understanding, wang19generalisation}, this paper assumes the \emph{episode independence}, defined as below,

\begin{assumption}[episode independence]
\label{mass:episode-independence}
For any $t=1,\dots,T$, state $s^t_h$ is sampled from a distribution $\mathcal{D}^{\pi_t}_{h}$, i.e.,
\(
    s^t_h \sim \mathcal{D}^{\pi_t}_{h}.
\)
The variables $\mathbf s_h^{1:T}=\{s^t_h\}_{t=1}^T$ are independent, but not necessarily identically distributed.
\end{assumption}

In this setting, the objective of a reinforcement learning algorithm is to find an optimal policy $\pi^*$ that maximises the expected rewards over the trajectory %from the state \(s_h\) {%under transition kernel $p^\dagger$ 
in deployment:
\begin{align*}
\pi^* =
\arg\max_{\pi\in\Pi}
\mathbb{E}_{s_h \sim \mathcal D_{h}^{\pi,\dagger}}
\Big[
V_h^{\pi,\dagger}(s_h)
\Big].
\end{align*}

This paper employs Wasserstein distance \cite{Kantorovich1960MathematicalMO, villani2008optimal} to measure environment shifts. % is a metric widely used for measuring distribution distance~
\begin{definition}[$p$-Wasserstein distance]
\label{def:wasserstein}
Let $\mathcal{S} \subseteq \mathbb{R}^d$ be a metric space equipped with a distance
function $d$. $\mu$ and $\nu$ are two probability measures on $\mathcal{S}$.
For any $p \ge 1$, the \emph{$p$-Wasserstein distance} between $\mu$ and $\nu$ is defined to be
\[
W_p(\mu,\nu)
\triangleq
\left(
\inf_{\gamma}
\int_{\mathcal{S}\times\mathcal{S}}
d(x,y)^p \,\mathrm{d}\gamma(x,y)
\right)^{1/p},
\]
where the infimum is taken over all joint distributions $\gamma$ on
$\mathcal{S}\times\mathcal{S}$ whose marginals coincide with $\mu$ and $\nu$.
\end{definition}

\begin{comment}
{
%We measure the discrepancy between the training and inference environments using 
Specifically, the $1$-Wasserstein distances between initial state distributions and transition kernels are as below,
\[
W_1(p_0^\dagger,p_0)
\triangleq
\sup_{s\in\mathcal S}
W_1\bigl(p_0^\dagger(s),p_0(s)\bigr).
\]
\[
W_1(p^\dagger,p)
\triangleq
\sup_{(s,a)\in\mathcal S\times\mathcal A}
W_1\bigl(p^\dagger(\cdot\mid s,a),p(\cdot\mid s,a)\bigr).
\]
}
\end{comment}

We employ the Total Variation (TV) distance to measure the distance between policies {\citep{Boucheron13concentration}}.
\begin{definition}[Total Variation (TV) distance]
The TV distance between two distributions $\mu, \nu$ is defined as,
{\begin{equation*}
\label{eq:tvdis}
d_\Pi(\mu,\nu) \triangleq \frac12 \sum_{x\in\mathcal{X}}|\mu(x)-\nu(x)|.
\end{equation*}}
% \end{definition}}
% \begin{equation*}
% \label{eq:tvdis}
% d_\Pi(\mu,\nu) \triangleq \sup_{x\in\mathcal X}\frac12|\mu(x)-\nu(x)|.
% \end{equation*}
\end{definition}
% \begin{definition}[Total Variation (TV) distance]
% The TV distance between policies $\pi, \pi' \in \Pi$ is defined as,
% \begin{equation*}
% \label{eq:tvdis}
% d_\Pi(\pi,\pi') \triangleq \sup_{s\in\mathcal S}\frac12\sum_{a\in\mathcal A}|\pi(a\mid s)-\pi'(a\mid s)|.
% \end{equation*}
% \end{definition}
The TV distance can be controlled by the Kullback-Leibler (KL) divergence {\citep{pinsker1964information}}.
\begin{definition}[Kullback-Leibler (KL) divergence]
The KL divergence between two distributions $\mu, \nu$ is defined as,
\begin{equation*}
\label{eq:tvdis}
\mathrm{KL}(\mu\|\nu) \triangleq \sum_{x\in\mathcal X}
\mu(x)\log\frac{\mu(x)}{\nu(x)}.
\end{equation*}
\end{definition}

\paragraph{Hypothesis Complexity}
Let the class of value functions be
\(
\mathcal Q_\Pi
\triangleq
\bigl\{\, s \mapsto Q_h^*(s,a_h^\pi): \pi\in\Pi,\ h\in[H] \bigr\}.
\) For brevity, we use $f(s)=Q_h^*(s,a_h^\pi) \triangleq Q_h^{\pi^*}(s,a_h^\pi)$. %for any $f\in\mathcal Q_\Pi$. 
Rademacher complexity, and its empirical version \citep{Bartlett2003RademacherAG,liu2022understanding}, are employed to measure the hypothesis complexity. %of a  space in the `independently' (not necessarily `identically') sampling setting 
{We define the empirical version here, and present Rademacher complexity in Appendix \ref{addsec:ige}.
% , while sequential Rademacher complexity is mostly for sequential, distribution-dependent prediction \cite{rakhlin2015sequentialcomplexities}.

\begin{definition}[empirical Rademacher complexity]
Let $\mathcal{F} \subseteq \mathbb{R}^{\mathcal{S}}$ be a function class and $\mathbf s^{1:n}=\{s^i\}_{i=1}^n$ be a sample set.
Let $\bm{\sigma}^{1:n}=(\sigma^1,\dots,\sigma^n)$ be independent Rademacher random variables.
The \emph{empirical Rademacher complexity} of $\mathcal{F}$ on the sample set $\mathbf s^{1:n}$ is defined as
\[
\hat{\mathfrak{R}}(\mathcal{F},\mathbf s^{1:n})
\triangleq
\frac{1}{n} 
\mathbb{E}_{\bm{\sigma}^{1:n}}
\left[
    \sup_{f\in\mathcal F}
    \sum_{i=1}^n \sigma^i f(s^i)
\right].
\]
\end{definition}}

\section{Rationality Measures}% Rational Value Loss and Regret}
\label{sec:rationality}

This section defines a suite of rationality measures for reinforcement learning. 
{To note, we are particularly interested in the rationality in deployment.} Intuitively, training processes usually employ gradient descent, or its variants, that optimise objectives in the steepest direction. In light of this, the rationality in training can be characterised by the discrepancy between the hidden actual value function and the objectives in optimisation. %, which is thus rational to a certain extent.} 

\subsection{Rationality Measures}% ity of Policies} %-- A ‘Local and Immediate’ Perspective}

We first define \emph{perfectly rational actions}. For brevity, we also call them \emph{rational actions} if no ambiguity is caused.

% \begin{definition}[rational action]
% An action \( a_h^{\pi} \) is called a \textit{rational action} if it maximises the hidden value function: %{it predicts the same long-term reward as the action generated from the hidden optimal value function \( Q_h(\pi^*;p^\dagger,s_h,a)\) i.e.}, 
% % \[
% %  \arg \max_{a \in \mathcal A} Q_h(\hat{\pi}_t;p,s_h, a) = \arg \max_{a \in \mathcal A}Q_h(\pi^*;p^\dagger,s_h, a),
% % \]
% \[
%  a_h^{\pi} \triangleq \arg \max_{a \in \mathcal A}Q_h(\pi^*;p^\dagger,s_h, a),
% \]
% where \(\pi \) is learned policy, $\pi^*$ is the optimal policy (so $Q_h(\pi^*;p^\dagger,s_h, a)$ is the hidden optimal value function), \( h \) is time step, \(s_h\) is state.
% \end{definition}

\begin{definition}[perfectly rational action]
An action \( a_h^{\circ} \) is called \textit{perfectly rational}, if its policy maximises the true value function of the state $s_h$ at the time step \( h \):
\[
 a_h^{\circ} \sim \pi^{\circ}(\cdot\mid s_h);\quad \pi^{\circ} \triangleq \arg \max_{\pi \in \Pi}\mathbb E_{a_h \sim\pi}\left[Q_h^{*,\dagger}(s_h,a_h)\right].
\]
% \bry{
% \[
%  a_h^{\circ} \in \arg \max_{a \in \mathcal A} \pi^{\circ}(a\mid s_h);\quad \pi^{\circ} \triangleq \arg \max_{\pi \in \Pi}\mathbb E_{a_h \sim\pi}\left[Q_h^{*,\dagger}(s_h,a_h)\right].
% \]
% }
 %$\pi^*$ is the optimal policy (so 
%$\mathbb E_{a_h \sim\pi}\left[Q_h^{*,\dagger}(s_h,a_h)\right]$ is the (unknown) true value function, 

\end{definition}

\begin{remark}
    As mentioned, the main goal is to study the rationality in deployment, which is not episodic.
\end{remark}

A reinforcement learning agent may be of \emph{bounded rationality}; i.e., the agent does not always take rational actions. %-- the actions could exhibit a gap from their rational counterparts. 
This incurs \emph{rational value loss}, defined as below.

\begin{definition}[rational value loss]
\label{def:rvl}
%Let $\pi^*$ be the optimal policy and l
Let $p^\dagger,p_0^\dagger$ denote the transition kernel and initial state distribution of the inference environment. %For any policy $\pi\in\Pi$, %and any step $h\in\{1,\dots,H\}$, 
The rational value loss of {the taken action $a^\pi_h$ drawn} from policy $\pi\in\Pi$ at step $h$ is defined as:
\begin{align*}
\mathcal{L}(a^{\hat{\pi}}_h, s_h)
&\triangleq
{Q_h^{*,\dagger}(s_h,a^{\circ}_h)}-Q_h^{*,\dagger}(s_h,a^\pi_h).
\end{align*}
%where $a^\pi_h$ is drawn from policy  $\pi(\cdot|s_h)$.
\end{definition}

{
\begin{remark}%[differences from the advantage function] 
%To interpret our definition, we contrast it with the widely used 
Compared with the advantage function $A_h^\pi(s_h,a_h) = Q_h^\pi(s_h,a_h) - V_h^\pi(s_h)$ \cite{schulman2017ppo},
our definition adopts a behavioural perspective, comparing the action $a_h^\pi$ with the perfectly rational action $a_h^\circ$. In contrast, the advantage function compares an action relative to the expectation over the action distribution.
\end{remark}
}

From the definitions, we directly prove Lemma \ref{lemma:trivial_case}.
\begin{lemma}
\label{lemma:trivial_case}
    If an action is perfectly rational, its rational value loss is zero.
\end{lemma}

% \bry{
% \begin{definition}[rational value loss]
% \label{def:rvl}
% %Let $\pi^*$ be the optimal policy and l
% Let $p^\dagger$ denote the transition kernel of the inference environment. %For any policy $\pi\in\Pi$, %and any step $h\in\{1,\dots,H\}$, 
% The rational value loss of policy $\pi\in\Pi$ at step $h$ is defined as:
% \begin{align*}
% \mathcal{L}_h(\pi, s_h)
% &\triangleq
% V_h^{*,\dagger}(s_h)-V_h^{\pi,\dagger}(s_h).
% \end{align*}
% where $a^\pi_h$ denote $a_h\sim \pi(\cdot|s_h)$.
% \end{definition}
% }

% \begin{definition}[trajectory rational value loss]
% \label{def:trvl}
% %Let $\pi^*$ be the optimal policy and let $p^\dagger$ denote the transition kernel of the inference environment. 
% For any policy $\pi\in\Pi$ that generates trajectory $\mathbf s_h^t=(s_h^t,\cdots,s_{H}^t)$, the trajectory rational value loss of policy $\pi$ over the trajectory $\mathbf s_h^t$ %and $H$ steps 
% is defined as:
% \[
% \mathcal{L}(\pi)
% \triangleq 
% \sum_{h=1}^H \left[V_h^{*,\dagger}(s_h)-V_h^{\pi,\dagger}(s_h)\right].
% \]
% \end{definition}

We then define \emph{expected rational value loss} in deployment. %in inference and the empirical rational value loss in training.
{
\begin{definition}[expected rational value loss]
\label{def:ervl}
Given the state distribution in deployment $\mathcal D_h^{\pi,\dagger}$ 
induced by policy $\pi \in \Pi$, the expected rational value loss of policy $\pi$ at time step $h$ is defined as:
\[
\mathcal{R}_h(\pi)
\triangleq
\mathbb E_{s_h \sim\mathcal D_h^{\pi,\dagger}}
\left[
Q_h^{*,\dagger}(s_h,a^{\circ}_h)-Q_h^{*,\dagger}(s_h,a^\pi_h)\right].
\]
\end{definition}
}

% \[
% \mathcal{R}(\pi)- \hat{\mathcal{R}}(\pi)
% \leq
% \sum_{h=1}^H\mathbb E_{s_h \sim\mathcal D_h^{\pi,\dagger}}\mathbb E
% \left[
% Q_h^{*,\dagger}(s_h,a^{\circ}_h)-Q_h^{*,\dagger}(s_h,a^\pi_h)\right]-
% \frac{1}{T}\sum_{t=1}^T \sum_{h=1}^H
% \left[{Q_h^{T}(s^t_h,a^{\circ,t}_h)}-{Q_h^{T}(s^t_h,a^{\pi,t}_h)}\right] \leq 2\sup_{\pi \in \Pi}\left[ \sum_{h=1}^H \mathbb E_{s_h \sim\mathcal D_h^{\pi,\dagger}}\mathbb E[Q_h^{*,\dagger}(s_h,a^{\pi}_h])-\frac{1}{T}\sum_{t=1}^T \sum_{h=1}^H{Q_h^{T}(s^t_h,a^{\pi,t}_h)}\right]
% \]

% $\sum_{h=1}^H\mathbb E_{s_h \sim\mathcal D_h^{\pi,\dagger}}\mathbb E
% \left[
% Q_h^{*,\dagger}(s_h,a^{\circ}_h)-Q_h^{*,\dagger}(s_h,a^\pi_h)\right]=\sum_{h=1}^H\mathbb E_{s_h \sim\mathcal D_h^{\pi,\dagger}}\mathbb E_{a_h \sim \pi^\circ}
% \left[
% Q_h^{*,\dagger}(s_h,a_h)\right]-\mathbb E_{a_h \sim \pi}\left[Q_h^{*,\dagger}(s_h,a_h)\right]$
% \bry{
% \begin{definition}[expected rational value loss]
% \label{def:ervl}
% Let $\mathcal D_h^{\hat{\pi},\dagger}$ denote the state distribution at step $h$
% induced by the optimal policy $\pi^*$ and the transition kernel $p^\dagger$ of the inference environment. For any policy $\pi\in\Pi$ and any step $h\in\{1,\dots,H\}$, the expected rational value loss of policy $\pi$ is defined as:
% \[
% \mathcal{R}_h(\pi)
% \triangleq
% \mathbb E_{\mathcal D_h^{\hat{\pi},\dagger}}
% \left[
% V_h^{*,\dagger}(s_h)-V_h^{\pi,\dagger}(s_h)\right].
% \]
% where $\mathbb E_{\mathcal D_{h}^{\pi}}[\cdot]$ denotes $\mathbb E_{s_h \sim \mathcal D_{h}^{\pi}}[\cdot]$
% \end{definition}
% }

The inference environment is supposed to be unknown; thus, the expected rational value loss is usually inaccessible. We then define an empirical version in training.

\begin{definition}[empirical rational value loss]
\label{def:arvl}
Suppose an agent is trained by \( T \) episodes, taking a sequence of states $\{s^t_h\}_{t=1}^T$. Let $p$, $p_0$ denote the transition kernels and initial state distribution of the training environment. %For any policy $\pi\in\Pi$, t
The empirical rational value loss of a policy $\pi$ at time step $h$ is defined as the average over the $T$ episodes, as below,
\begin{align*}
\hat{\mathcal{R}}_h(\pi)
\triangleq
\frac{1}{T}\sum_{t=1}^T
\left[{Q_h^{*}(s^t_h,a^{\circ}_h)}-Q_h^{*}(s^t_h,a^{\pi}_h)\right].
\end{align*}
\end{definition}

% \bry{
% \begin{definition}[empirical rational value loss]
% \label{def:arvl}
% For a total of \( T \) episodes, let $\{s^t_h\}_{t=1}^T$ denote the sequence of states encountered during training and let $p$ denote the transition kernel of the training environment. For any policy $\pi\in\Pi$, the empirical rational value loss at step $h$ over $T$ episodes is defined as
% \begin{align*}
% \hat{\mathcal{R}}_h(\pi)
% \triangleq
% \frac{1}{T}\sum_{t=1}^T\left[
% V_h^{*}(s^t_h)-V_h^{\pi}(s^t_h)
% \right].
% \end{align*}
% \end{definition}
% }

By cumulating the {expected and empirical} rational value loss over a trajectory, %during inference and training, 
we define \emph{expected rational value risk} and \emph{empirical rational value risk} as follows.

{
\begin{definition}[expected rational value risk]
\label{def:etrvl}
Let $\mathcal D_h^{\pi,\dagger}$ denote the state distribution in deployment, at step $h$,
induced by a policy $\pi \in \Pi$, transition kernels $p^\dagger$, initial state distribution $p^\dagger_0$. %, and initial state $s_1 \sim p^\dagger_0$. 
The expected rational value risk of the policy $\pi$ over a trajectory of horizon $H$ is defined as:
\[
\mathcal{R}(\pi)
\triangleq
\sum_{h=1}^H\mathbb E_{s_h \sim \mathcal D_h^{\pi,\dagger}}\left[{Q_h^{*,\dagger}(s_h,a^{\circ}_h)}-Q_h^{*,\dagger}(s_h,a^\pi_h)\right].
\]
\end{definition}
}
% \[
% \left|\mathcal{R}(\pi)-\hat{\mathcal{R}}(\pi)\right|
% \triangleq
% \left|\sum_{h=1}^H\mathbb E_{s_h \sim \mathcal D_h^{\hat{\pi},\dagger}}\left[{Q_h^{*,\dagger}(s_h,a^{\circ}_h)}-Q_h^{*,\dagger}(s_h,a^\pi_h)\right]-\frac{1}{T}\sum_{t=1}^T \sum_{h=1}^H
% \left[Q_h^{*}(s_h^t,a^{\circ}_h)-Q_h^{*}(s_h^t,a^\pi_h)
% \right]\right|.
% \]

\begin{definition}[empirical rational value risk]
\label{def:atrvl}
Suppose an agent is trained by \( T \) episodes, each of horizon \( H \). %, realising trajectory $\{\mathbf s^t_h\}_{t=1}^T$. 
Let $p,p_0$ denote the transition kernel and initial state distribution of the training environment. The empirical rational value risk of policy $\pi$ is defined as the average over $T$ episodes:
\[
\hat{\mathcal{R}}(\pi)
\triangleq
\frac{1}{T}\sum_{t=1}^T \sum_{h=1}^H
\left[Q_h^{*}(s_h^t,a^{\circ}_h)-Q_h^{*}(s_h^t,a^{\pi}_h)
\right].
\]
\end{definition}

The gap between the expected and empirical rational value risks, termed {\it rational risk gap}, reflects how rational an agent is in deployment, given the behaviour in training. %We develop theory mainly for this gap. %deviations from learned rational actions to deployment.} %We measure her rationality by defining the {\it rational risk gap} as \(\left|\mathcal{R}(\pi)-\hat{\mathcal{R}}(\pi) \right|\).

We also define the \emph{asymptotic rational risk gap} as below. It helps understand the asymptotic property of an agent in terms of rationality. %It guarantees that the reinforcement learning agent can eventually become rational after a long run. 
\begin{comment}
    
\begin{definition}[asymptotic rationality]
An agent is defined to be asymptotically rational if, %the probability of the event that the supremum of rational risk gap $\sup_{\pi \in \Pi}\left|\mathcal{R}(\pi)-\hat{\mathcal{R}}(\pi)\right|$ will converge to zero f
for any $\epsilon >0$, we have,
\[
%\forall \epsilon >0, \quad 
\lim_{T \to \infty} \Pr \left\{ \sup_{\pi \in \Pi} \left| \mathcal{R}(\pi) - \hat{\mathcal{R}}(\pi)\right| >\epsilon \right\}=0.
\]
% where $V^\dagger_h$ presents the inference value of $\mathbb{E}_{s_h\sim\mathcal D_h^{\hat{\pi},\dagger}} V_h^{\pi,\dagger}(s_h)$, and $\tilde{V}_h$ denotes the on-average training value of $\frac{1}{T}\sum_{t=1}^T
% \mathbb{E}_{t-1}
% \!\left[
% V_h^{\pi}(s^t_h)
% \right].$ 
\end{definition}
\end{comment}

\begin{definition}[asymptotic rational risk gap]
\label{def:asy_gap}
%Under the same conditions of Definitions \ref{def:etrvl} and \ref{def:atrvl}, w
The asymptotic rational risk gap is defined to be
\(
\lim_{T \to \infty} \left| {\mathcal{R}}(\pi) - \hat{\mathcal{R}}(\pi) \right|
\).
\end{definition}

\paragraph{``Local and Immediate'' Perspective of Rationality} This paper takes a ``local and immediate'' perspective for defining rationality measures. For example, an action is defined to be rational if this {\it individual} move is optimal in terms of the {value function}, %$\mathbb{E}[Q_h^{*,\dagger}(s_h,a^\pi_h)]$}, given all information {\it accessible to the agent} \
{\it at the time}. %Similarly, the rational value loss measures how rational an agent is for taking a single action in terms of immediate action value. 
In other words, a rational agent is not expected to have the capabilities of overlooking the global landscape or anticipating the future, in a ``global and long-term'' view. We appreciate that such a more strategic perspective of rationality is also valuable, which is, however, out of the scope of this paper.

\paragraph{``Behavioural and Objective'' Perspective of Rationality} 

Our measure characterises how well an agent is maximising its utility, which coincides with the economic papers \cite{stephen80, HALPERN2001425}, while some others define rationality subject to the restrictive access to information and the uncertainty in decision making \cite{Dean1993}. Under the subjective expected utility framework, an agent may be considered rational relative to her subjective beliefs even when her actions are objectively suboptimal \cite{fishburn1981subjective}.  This paper takes an objective perspective %that the agent behaviour derived from its fixed policy in deployment in terms of maximising utility rather than agent capability, for 
based on a two-fold rationale: (1) in modern practice, such as in LLMs, the training size is massive, and the training process is mostly black-box, so looking at the capabilities could be intractable; and (2) having a measure of agent behaviour is more direct and helpful for understanding interactions between agents and humans, and thus for understanding agents’ impact on society. %Compared with existing related works, this rationality measure is more applicable to general reinforcement learning agents in deployment with tractable upper bounds, providing actionable implications for improving agent rationality.

\subsection{Decomposition of Rational Risk Gap}

We now present a lemma on the decomposition of the rational risk gap, %$\left|\mathcal R(\pi)-\hat{\mathcal{R}}(\pi)\right|$, 
which indicates the sources of sub-rationality. %Formally, we present the following lemma.

{
\begin{lemma}[decomposition of rational risk gap]
\label{lem:decomp-etrvl-atrvl}
The rational risk gap $\left|\mathcal R(\hat{\pi})-\hat{\mathcal{R}}(\hat{\pi})\right|$ of policy $\hat{\pi} \in \Pi$ over a trajectory of horizon $H$ can be decomposed as follows,
\begin{align*}
&\left|\mathcal{R}(\hat{\pi})-\hat{\mathcal{R}}(\hat{\pi})\right|\leq\\
%&\leq
%2 \sum_{h=1}^H \sup_{\pi \in \Pi} \left|\mathbb E_{s_h \sim \mathcal D_h^{\hat{\pi},\dagger}} Q_h^{*,\dagger}(s_h,a_h^\pi) - \frac{1}{T}\sum_{t=1}^T Q_h^{*}(s^t_h,a_h^\pi)\right|\\
&
2 \sum_{h=1}^H\underbrace{
\sup_{\pi \in \Pi}
\left|
\mathbb E_{s_h \sim \mathcal D_h^{\hat{\pi},\dagger}}
Q_h^{*,\dagger}(s_h,a_h^\pi)
-
\mathbb E_{s_h \sim \mathcal D_h^{\hat{\pi}}}
Q_h^{*}(s_h,a_h^\pi)
\right|
}_{\text{extrinsic rational gap}} \\
&+
2 \sum_{h=1}^H\underbrace{
\sup_{\pi \in \Pi}
\left|
\mathbb E_{s_h \sim \mathcal D_h^{\hat{\pi}}}
Q_h^{*}(s_h,a_h^\pi)
-
\frac{1}{T}\sum_{t=1}^T Q_h^{*}(s^t_h,a_h^\pi)
\right|
}_{\text{intrinsic rational gap}},
\end{align*}
where $\mathcal D_{h}^{\hat{\pi}}$ is the state distribution in training induced by policy $\hat{\pi}$, %with initial state distribution $p_0$ and transition kernel $p$, 
while $\mathcal D_h^{\hat{\pi},\dagger}$ is the state distribution in deployment induced by the same policy but under different transition kernels %$p^\dagger$ 
and initial state distribution. %$p_0^\dagger$.
\end{lemma}
}

This lemma suggests that the rational risk gap can be decomposed into two components as follows:
%\begin{itemize}
    % \item  \textbf{Sequential prediction error}: the discrepancy between the on-average training value \(\frac{1}{T}\sum_{t=1}^T \mathbb{E}_{i-1}\!\left[V_h^{\pi}(s^t_h)\right] \) and its empirical estimation \( \frac{1}{T}\sum_{t=1}^T V_h^{\pi}(s^t_h) \) by the sequential prediction. This error is determined by the online learnability of a function class under distribution-dependent sampling. It quantifies how difficult it is for an agent to learn a policy when state samples are generated from distributions depend on past observations as well as previously updated policies.
%    \item  
\paragraph{Extrinsic Rational Gap:} the distance between the true value in deployment, $\mathbb{E}_{s_h\sim\mathcal D_h^{\hat{\pi},\dagger}} Q_h^{*,\dagger}(s_h,a_h^\pi)$, and its counterpart in training, $\mathbb{E}_{s_h \sim \mathcal D_h^{\hat{\pi}}}Q_h^{*}(s_h,a_h^\pi)$. Intuitively, it arises from the \emph{training-to-deployment shifts}, closely linking to %{which may help understand the 
the more well-known \emph{sim-to-real challenge}} {\citep{peng2018sim,tobin2017domainrandomization,andrychowicz2020learning}}. Specifically, changes of the transition kernel ($p$ to $p^\dagger$) and of initial state distribution ($p_0$ to $p_0^\dagger$) induce different state distributions ($\mathcal D_h^{\hat{\pi}}$ vs. $\mathcal D_h^{\hat{\pi},\dagger}$), and hence different optimal value functions ($Q_h^{*}(s_h,a_h^\pi)$ vs. $Q_h^{*,\dagger}(s_h,a_h^\pi)$). 
    %\item  
    
\paragraph{Intrinsic Rational Gap:} the difference between the expected value $\mathbb{E}_{s_h \sim \mathcal D_h^{\hat{\pi}}}Q_h^{*}(s_h,a_h^\pi)$ and its empirical version \( \frac{1}{T}\sum_{t=1}^T Q_h^{*}(s_h^t,a_h^\pi) \), both in training. %at each episode $1\leq t\leq T$. 
This gap is determined by the joint effects of generalisability and the online setting of reinforcement learning, reflecting the capacity to learn the optimal policy in a dynamic environment. %, as well as the state distribution shift induced by iterated policy updates.}

\section{Rationality Theory}

This section develops theory for the rational risk gap. The theory relies on the following assumptions. %on (1) the smoothness of value functions.

%measured by % We introduce 
%the following integral probability metric (IPM). %to measure the shifts.

%\begin{definition}[integral probability metric (IPM)]
%\label{mdef:ipm}
%{Let $\mathcal{Q}_\Pi\subseteq\{f:\mathcal{S}\to\mathbb{R}\}$ be a class of bounded measurable functions. For any probability measures $\mu, \nu$ on $\mathcal{S}$, the IPM of $\mathcal{Q}_\Pi$} is defined as
%\[
%    D_{\mathcal{Q}_\Pi}(\mu,\nu)
%    \triangle
%    \sup_{f \in \mathcal{Q}_\Pi}
%    \bigl|
%        \mathbb{E}_{\mu}[f]
%        -
%        \mathbb{E}_{\nu}[f]
%    \bigr|.
%\]
%\end{definition}

\begin{assumption}[Lipschitz-continuous value]
\label{mass:lipschitz-val}
%For any $h\in\{1,\dots,H\}$, %and fix a policy $\pi\in\Pi$, 
$L_s$ is a positive constant. We assume that the value function $f\in \mathcal{Q}_\Pi$ is $L_s$-Lipschitz under distance function $d$, for any policy $\pi \in \Pi$, $h\in\{1,\dots,H\}$ and $s,\tilde s\in\mathcal S$,
\[
|f(s)-f(\tilde s)|
\le L_s d(s,\tilde s).
\]
\end{assumption}

%In addition, we assume that small perturbations in the environment dynamics do not lead to large changes in the state distributions, as follows:
{
\begin{assumption}[Lipschitz-continuous transition]
\label{mass:lipschitz-env}
$L_p$ is a positive constant. For any $h \in \{1, \ldots, H\}$, let $W_1(p^\dagger, p)\triangleq \sup_{s\in\mathcal{S},a\in \mathcal{A}}W_1(p^\dagger(\cdot|s,a), p(\cdot|s,a))$ be the 1-Wasserstein distance between state distributions induced by $p^\dagger$ and $p$ for any $s \in \mathcal{S},a\in\mathcal{A}$. For any $\pi \in \Pi$, we assume the $W_1(p^\dagger, p)$ changes by at most $L_p\, W_1(p^\dagger, p)$ at each time step,
\[
W_1\!\left(\mathcal D_{h+1}^{\pi,\dagger},\mathcal D_{h+1}^{\pi}\right)
\le
W_1\!\left(\mathcal D_{h}^{\pi,\dagger},\mathcal D_{h}^{\pi}\right)
+
L_p\,W_1\!\left(p^\dagger,p\right).
\]
\end{assumption}
}

% {\begin{assumption}[Lipschitz-continuous transition kernel]
% \label{mass:lipschitz-env}
% $L_p$ is a positive constant. We assume that the mapping from transition kernel to the induced state distribution is $L_p$-Lipschitz under the $1$-Wasserstein distance for any $\pi\in\Pi$,
% \[
%     W_1\bigl(\mathcal D_h^{\pi,\dagger},\mathcal D_{h}^{\pi}\bigr)
%     \le L_{p} W_1(p^\dagger,p).
% \]
% \end{assumption}

\begin{assumption}[Lipschitz-continuous policy]
  \label{mass:lipschitz-policy}
$L_\Pi$ is a positive constant. {Let $d_\Pi(\pi,\pi')\triangleq\sup_{s\in\mathcal{S}} d_\Pi(\pi(s),\pi'(s))$ be the TV distance between $\pi,\pi'$ for any $s\in\mathcal S$.} For any $\pi, \pi' \in \Pi$, we assume that the mapping $\pi \mapsto \mathcal{D}_h^{\pi}$ is
  $L_\Pi$-Lipschitz under the TV distance $d_{\Pi}$, 
  \[\sup_{f\in\mathcal{Q}_\Pi}
    \bigl|
        \mathbb{E}_{\mathcal D_h^\pi}[f]
        -
        \mathbb{E}_{\mathcal D_h^{\pi'}}[f]
    \bigr|
    \le
    L_\Pi\, d_{\Pi}(\pi,\pi').\]
\end{assumption}

\begin{assumption}[Entropy-regularised policy]
\label{mass:klpolicy}
We assume that the learned policy has the following {KL bound}:
\[
\sup_{s\in\mathcal S}
\mathrm{KL}(\pi_{t+1}(\cdot\mid s)\|\pi_t(\cdot\mid s))
\le \alpha.\]
\end{assumption}

{
\begin{remark}
These Assumptions are reasonably mild, following \citet{bukharin2023,gottesman23a, wang19generalisation, schulman2018equivalence, nino2020}. {They mean that (1) environments are smooth with respect to (w.r.t.) states, (2) the learned policy satisfies the smoothness condition, and (3) the learned policy does not go too far away.} %\fh{give references for KL bound}  %} show that under smooth environment and policy setting, the value functions are also smooth, as stated in Assumption \ref{mass:lipschitz-val}. Similar smoothness assumptions have also been adopted in existing works \cite{
\end{remark}
}

\subsection{Extrinsic Rational Gap Bound}

We first study the extrinsic rational gap.

\begin{theorem}[extrinsic rational gap bound]
\label{mthm:bias-unified}
Let $\mathcal D_h^{\hat{\pi},\dagger}, \mathcal D_h^{\hat{\pi}}$ denote the state distributions in inference and training, respectively. Under Assumptions~\ref{mass:lipschitz-val}--\ref{mass:lipschitz-env}, the extrinsic rational gap over a trajectory of horizon $H$ is  upper bounded by
\begin{align*}
&\sum_{h=1}^H \sup_{\pi \in \Pi} \left|\mathbb{E}_{s_h\sim\mathcal D_h^{\hat{\pi},\dagger}} Q_h^{*,\dagger}(s_h,a_h^{\pi})
-
 \mathbb{E}_{s_h \sim \mathcal D_h^{\hat{\pi}}}Q_h^{*}(s_h,a_h^\pi)\right| \\
&\leq L_s H\cdot W_1(p_0^\dagger,p_0)+H^2L_s(L_p+1)\cdot W_1(p^\dagger, p).
\end{align*}
\end{theorem}

This theorem shows that the extrinsic rational gap is determined by (1) $L_sH\cdot W_1(p_0^\dagger,p_0)$ that arises from the discrepancy between initial state distributions of $p_0^\dagger$ and $p_0$, and (2) $H^2L_s(L_p+1)\cdot W_1(p^\dagger, p)$ caused by the difference between transition kernels of $p^\dagger$ and $p$.

A detailed proof is given in Appendix \ref{addsec:ege}. 

\paragraph{Proof Sketch}
We first decompose the extrinsic rational gap (ERG) at time step $h\in [H]$ into two terms as follows,
\begin{align*}
&\sup_{\pi \in \Pi} \left|\mathbb{E}_{s_h\sim\mathcal D_h^{\hat{\pi},\dagger}} Q_h^{*,\dagger}(s_h,a_h^\pi)
-
\mathbb E_{s_h \sim \mathcal D_h^{\hat{\pi}}} Q_h^{*}(s_h,a_h^\pi)\right| \\
&
{\leq} 
\underbrace{
\sup_{\pi \in \Pi} \left|\mathbb{E}_{s_h\sim\mathcal D_h^{\hat{\pi},\dagger}} Q_h^{*,\dagger}(s_h,a_h^\pi)
-
\mathbb{E}_{s_h\sim\mathcal D_h^{\hat{\pi}}} Q_h^{*,\dagger}(s_h,a_h^\pi)\right|}_{\text{I}} \\
&\quad+ 
\underbrace{
\sup_{\pi \in \Pi} \left|\mathbb{E}_{s_h\sim\mathcal D_h^{\hat{\pi}}} Q_h^{*,\dagger}(s_h,a_h^\pi)
-
\mathbb{E}_{s_h\sim \mathcal D_h^{\hat{\pi}}} Q_h^{*}(s_h,a_h^\pi)\right|}
_{\text{II}}.
\end{align*}

Term I arises from the distance between the state distributions $\mathcal D_h^{\hat{\pi},\dagger}$ and $\mathcal D_h^{\hat{\pi}}$. Under Assumption~\ref{mass:lipschitz-env}, this term admits the upper bound $\text{Term I} \leq  W_1(p_0^\dagger,p_0)+
    (h-1)L_{p}\, W_1(p^\dagger,p)$. %quantifies discrepancies ca in true action value under the same state distribution between different transition kernels. 
    Term II is shown in Lemma~\ref{mlem:trans_error} that scales linearly with the Wasserstein distance between $p$ and $p^\dagger$, with an additional dependence on the horizon.

\begin{lemma}
\label{mlem:trans_error}
Under Assumption \ref{mass:lipschitz-val}, for any step $h\in[H]$, the optimal value discrepancy between the inference transition kernel $p^\dagger$ and training transition kernel $p$, under same training distribution $\mathcal D_h^*$, satisfies
\begin{align*}
&\sup_{\pi \in \Pi} \left|
\mathbb E_{s_h\sim\mathcal D_h^{\hat{\pi}}}
Q_h^{*,\dagger}(s_h,a_h^\pi)
-
\mathbb E_{s_h\sim \mathcal D_h^{\hat{\pi}}}
Q_h^{*}(s_h,a_h^{\pi})
\right|\\
&\qquad \le
(H-h)\,L_s\,W_1(p^\dagger,p).
\end{align*}
\end{lemma}

Combining these two terms over a trajectory of horizon H, we obtain an upper bound on the extrinsic rational gap in Theorem~\ref{mthm:bias-unified}.

\subsection{Intrinsic Rational Gap Bound}

%To study the intrinsic rational gap, we make the following assumptions. 

We then obtain the following high-probability upper bound for the intrinsic rational gap.

\begin{theorem}[intrinsic rational gap bound]
  \label{mthm:multi-source-rl}
  Under Assumptions~\ref{mass:lipschitz-env},~\ref{mass:lipschitz-policy} and~\ref{mass:klpolicy}, let $\hat{\mathfrak{R}}_h(\mathcal{Q}_\Pi)$ denote the empirical Rademacher complexity of value function class $\mathcal{Q}_\Pi$ with a sequence of states $\mathbf s_h^{1:T}=\{s_h^t\}_{t=1}^T$ at time step $h\in [H]$. For any $\delta\in(0,1)$, with probability at least $1-\delta$, the upper bound on intrinsic rational gap is:
  \begin{align*}
      &\sum_{h=1}^H\sup_{\pi \in \Pi} \left| \mathbb{E}_{s_h \sim \mathcal D_h^{\hat{\pi}}} Q_h^*(s_h,a_h^{\pi})
      -
      \frac{1}{T}\sum_{t=1}^T Q_h^*(s_h^t,a_h^{\pi}) \right| \\
 &\leq 
    L_\Pi H \sqrt{2\log|\mathcal A|}+2\sum_{h=1}^H\hat{\mathfrak{R}}_h(\mathcal{Q}_\Pi)+3H^2\sqrt{\frac{\log(4H/\delta)}{2T}}.
  \end{align*}
\end{theorem}

%\begin{remark}
%The empirical Rademacher complexity $\hat{\mathfrak R}(\mathcal Q_\Pi)$ is often bounded as $\tilde O(m/T)$ for some complexity parameter $m$. However, this bound may be loose when $m$ is large. Sharp bounds can be obtained through localisation techniques \cite{bartlett2002localized} or algorithm-dependent Rademacher complexity \cite{sachs23a}.
%\end{remark}

% \bry{remark: explain why to choose the state within the time step and its bayes version}

This bound depends on the empirical Rademacher complexity $\sum_{h=1}^H\hat{\mathfrak{R}}_h(\mathcal Q_\Pi)$, which measures the capacity of the value function class under finite-sample training. The term $L_\Pi \sqrt{2\log|\mathcal A|}$ arises from policy shift between the initial uniform policy and the fixed policy $\hat{\pi}$, which scales with the logarithm of the action space cardinality $|\mathcal A|$. The remaining term is a concentration term that decays at a rate $O(T^{-1/2})$, as the number of training episodes increases.

A detailed proof is provided in Appendix~\ref{addsec:ige}.
{
\paragraph{Proof Sketch}
 We decompose the intrinsic rational gap %at time step $h \in [H]$ 
 into two terms:

\begin{align*}
&\sup_{\pi \in \Pi}\left|\mathbb{E}_{s_h\sim\mathcal D_h^{\hat{\pi}}} Q_h^{*}(s_h,a_h^\pi)
-
\frac{1}{T}\sum_{t=1}^T Q_h^{*}(s^t_h,a_h^\pi)\right| \\
&\le
\underbrace{
\sup_{\pi \in \Pi} \left|\mathbb{E}_{s_h\sim\mathcal D_h^{\hat{\pi}}} Q_h^{*}(s_h,a_h^\pi)
-
\frac{1}{T}\sum_{t=1}^T \mathbb{E}_{s_h^t \sim \mathcal D_{h}^{\pi_t}}Q_h^{*}(s_h^t,a_h^\pi)\right|}
_{\text{I}}\\
&\hspace{0.5cm}+\underbrace{
\sup_{\pi \in \Pi} \left|\frac{1}{T}\sum_{t=1}^T \left[\mathbb{E}_{s_h^t\sim\mathcal D_{h}^{\pi_t}} Q_h^{*}(s_h^t,a_h^\pi)
-
 Q_h^{*}(s_h^t,a_h^\pi)\right]\right|}
_{\text{II}}.
\end{align*}

{Term~I can be bounded by %\(\sum_{t=2}^{T} \sum^{t-1}_{i=1} d_\Pi(\pi_{i+1},\pi_i)/T \leq \sqrt{T^2 \alpha/8}\). Term I can be bounded in 
the following Lemma~\ref{mlem:policy_bound}.} 

% {Term~I is calculated from the policy drift between the state distribution $\mathcal D_h^{\hat{\pi}}$ induced by the optimal policy $\pi^*$ and the state distributions in training $\{\mathcal D_{h}^{\pi_t}\}_{t=1}^T$. By Assumption~\ref{mass:lipschitz-policy}, 
% we obtain %\(\sum_{t=2}^{T} \sum^{t-1}_{i=1} d_\Pi(\pi_{i+1},\pi_i)/T \leq \sqrt{T^2 \alpha/8}\). Term I can be bounded in 
% the following Lemma~\ref{mlem:policy_bound}.} 

\begin{lemma}[policy drift bound]
\label{mlem:policy_bound}
Under Assumptions~\ref{mass:lipschitz-policy} and~\ref{mass:klpolicy}, let $\mathcal A$ be a finite action space and $\pi \in \Pi$ be a policy. Set parameter $\alpha= 4\log{|\mathcal A|}/T^2$. At time step $h\in [H]$ over $T$ episodes, we have this policy drift bound,
\begin{align*}
    &\sup_{\pi \in \Pi} \left|\mathbb{E}_{s_h\sim\mathcal D_h^{\hat{\pi}}} Q_h^{*}(s_h,a_h^\pi)
    -
    \frac{1}{T}\sum_{t=1}^T \mathbb{E}_{s_h^t \sim \mathcal D_{h}^{\pi_t}}Q_h^{*}(s_h^t,a_h^\pi)\right|\\
    &\le
     L_\Pi \sqrt{2\log|\mathcal A|}.
  \end{align*}
\end{lemma}

% Term~II reflects the discrepancy between expected values under the policy-induced distributions $\mathcal D_{h}^{\pi_t}$ and their empirical counterparts in training. Under the assumption~\ref{mass:episode-independence}, it can be bounded by the following Lemma~\ref{mlem:rad-noniid-full}.

Then, we obtain the upper bound for Term~II. %in Lemma~\ref{mlem:rad-noniid-full}.

\begin{lemma}[on-average generalisation bound]
\label{mlem:rad-noniid-full}
Let $\mathbf s^{1:T}_h =\{s_h^1,\dots,s_h^T\}$ be independent
random variables with $s^t_h \sim \mathcal D_{h}^{\pi_t}$ on a space $\mathcal S$. Define the averaged state distribution
\(
    \bar{\mathcal D}_{h} \triangleq \frac{1}{T}\sum_{t=1}^T \mathcal D_{h}^{\pi_t},
\)
and the Rademacher complexity
\(
    \mathfrak{R}_h(\mathcal Q_\Pi)
\) of value function class $\mathcal Q_\Pi$. For any $\delta\in(0,1)$, with probability at least $1-\delta/2H$, we have: 
\begin{align*}
    &\sup_{\pi \in \Pi}\left[\mathbb E_{s_h\sim\bar{\mathcal{D}}_{h}}[Q_h^{*}(s_h,a_h^\pi)]
        -
        \frac{1}{T}\sum_{t=1}^T Q_h^{*}(s_h^t,a_h^\pi)\right]\\
        &\le
    2\mathfrak{R}_h(\mathcal Q_\Pi)
    +
    \sqrt{\frac{H^2\log(2H/\delta)}{2T}}.
\end{align*}
\end{lemma}

Combining the two lemmas over a trajectory of horizon $H$, we prove Theorem \ref{mthm:multi-source-rl}. %$ \text{IRG}\leq 

\subsection{Main Result}

We now obtain the main theorem on the rational risk gap bound directly from the two previous subsections.}

{
\begin{theorem}[rational risk gap bound]
\label{mthm:main_theory}
Under then same conditions of Theorems \ref{mthm:bias-unified} and \ref{mthm:multi-source-rl}, %Suppose the value function class $\mathcal{Q}_\Pi$ with a sequence of states $\mathbf s_h^{1:T}=\{s_h^t\}_{t=1}^T$ at time step $h\in [H]$ has empirical Rademacher complexity $\hat{\mathfrak{R}}_h(\mathcal{Q}_\Pi)$. 
for %any policy $\pi \in \Pi$ and 
any $\delta\in(0,1)$, with probability at least $1-\delta$, the rational risk gap of policy $\hat{\pi} \in \Pi$ over $T$ episodes of horizon $H$ can be bounded by:
\begin{align*}
&\bigl|\mathcal R(\hat{\pi})-\hat{\mathcal{R}}(\hat{\pi})\bigr|\leq
\beta_1\cdot W_1(p_0^\dagger,p_0)+\beta_2\cdot W_1(p^\dagger, p)\\
&+2L_\Pi H \sqrt{2\log|\mathcal A|}
+4\sum_{h=1}^H\hat{\mathfrak{R}}_h(\mathcal{Q}_\Pi)+6H^2\sqrt{\frac{\log(4H/\delta)}{2T}},
\end{align*}
where $\beta_1=2L_sH$ and $\beta_2=2H^2L_s(L_p+1)$. %; the $1$-Wasserstein distances $W_1(p^\dagger_0,p_0)$ and $W_1(p^\dagger,p)$, and Lipschitz constants $L_s$, $L_p$, and $L_\Pi$ are given above. %are given in Assumption~\ref{mass:lipschitz-val},~\ref{mass:lipschitz-env}, and \ref{mass:lipschitz-policy}. %is the Lipschitz constant of the mapping from a policy $\pi$ to its induced state distribution, and $|\mathcal A|$ denotes the cardinality of the action space.
\end{theorem}

{
The empirical rational risk defined in Definition~\ref{def:atrvl} requires access to the optimal action-value function $Q^*_h$, which might be unavailable in practice. To address this limitation, we extend the definition to a more general form as follows, which offers an empirical metric. %as rational value metric based on the algorithm by the following Definition~\ref{def:exemp}.
\begin{definition}[rational value metric]
\label{def:exemp}
Let $\hat{Q}^T_h$ be an approximate action-value function of any algorithm after $T$ episodes and horizon $h$. %empirical rational value risk of policy $\pi$ in Definition~\ref{def:atrvl} can be extended as the average over $T$ episodes:
Its rational value metric is defined as follows,
\[
\hat{\mathcal{R}}_{\mathrm{ALG}}(\pi)
\triangleq
\frac{1}{T}\sum_{t=1}^T \sum_{h=1}^H
\left[\max_{a \in\mathcal{A}}\hat{Q}_h^T(s_h^t,a)-\hat{Q}_h^{T}(s_h^t,a^\pi_h)
\right].
\]
\end{definition}

Correspondingly, Theorem~\ref{mthm:main_theory} leads to the following result. A detailed proof is provided in Appendix \ref{sec:cor}.
% \begin{corollary}[extension of rational value risk bound]
% \label{cor:extension}
% Suppose the empirical rational value risk measured by any empirical action-value function $\hat{Q}^T_h$, the rational risk gap $\left|\mathcal{R}(\pi)-\hat{\mathcal{R}}(\pi)\right|$ of policy $\pi \in \Pi$ over a trajectory of horizon $H$ can be decomposed as follows,
% \begin{align*}
% &\left|\mathcal{R}(\pi)-\hat{\mathcal{R}}_{\mathrm{ALG}}(\pi)\right|\leq \left|\mathcal{R}(\pi)-\hat{\mathcal{R}}(\pi)\right|+\\
% &
% 2 \sum_{h=1}^H\underbrace{
% \sup_{\pi \in \Pi}
% \left|
% \frac{1}{T}\sum_{t=1}^T \left[Q_h^{*}(s^t_h,a_h^\pi)-\hat{Q}_h^{T}(s^t_h,a_h^\pi)\right]
% \right|
% }_{\text{approximation error}},
% \end{align*}
% \end{corollary}}
% \bry{
% As a case study, we apply this extended framework to analyse the rational value risk bound of Deep Q-Network (DQN) with two-layer ReLU neural networks trivially adopting the setting in~\citet{liu2022understanding}.
\begin{corollary}[rational value metric bound]
\label{mlem:exdqn}
Assuming approximate value function $\hat{Q}_h^T$ approximates optimal value function $Q_h^*$ with a bounded error, $\|Q_h^*-\hat{Q}_h^T\|_{\infty}\leq \epsilon$ for any $h$. Then, for any $\delta \in (0,1)$, with probability at least $1-\delta$, we have %for policy $\hat{\pi} \in \Pi$ over $T$ episodes of horizon $H$ %can be bounded by:
\begin{align*}
&\bigl|\mathcal R(\hat{\pi})-\hat{\mathcal{R}}_{\mathrm{ALG}}(\hat{\pi})\bigr|\leq
\beta_1\cdot W_1(p_0^\dagger,p_0)+\beta_2\cdot W_1(p^\dagger, p)\\&
+4\sum_{h=1}^H\hat{\mathfrak{R}}_h(\mathcal{Q}_\Pi) + 6H^2\sqrt{\frac{\log(4H/\delta)}{2T}}+
2 H\epsilon\\
&+2L_\Pi H \sqrt{2\log|\mathcal A|},
\end{align*}
where $\beta_1=2L_sH$ and $\beta_2=2H^2L_s(L_p+1)$.
\end{corollary}}

{
\subsection{Rational Risk Gap Bound under Reward Shift}
In addition to environment shifts, the reward function may also change at deployment. We formalise this problem through the following assumption. 

\begin{assumption}[reward shift]
\label{ass:rsb}
    For any $\varphi \leq 1$, let $r_h(s,a)$ and $r'_h(s,a)$ denote the reward function in reference and training, respectively. We assume that for any $(s,a) \in \mathcal S \times \mathcal A$ and any $h \in H$, they satisfy
    $$|r_h(s,a) - r'_h(s,a)| \leq \varphi.$$
\end{assumption}

Under this assumption, the rational risk gap bound in Theorem~\ref{mthm:main_theory} can be extended to Corollary~\ref{cor:extension} in the reward shift setting. A detailed proof is presented in Appendix \ref{sec:cor}.

\begin{corollary}[rational risk gap bound under reward shift]
\label{mcor:extension}
Under Assumption~\ref{ass:rsb}, suppose $\mathcal{R}'(\pi)$ denotes the expected rational value risk under $r'_h(s,a)$, the rational risk gap $\left|\mathcal{R}'(\pi)-\hat{\mathcal{R}}(\pi)\right|$ of policy $\pi \in \Pi$ over a trajectory of horizon $H$ can be decomposed as follows,
\begin{align*}
&\bigl|\mathcal R'(\hat{\pi})-\hat{\mathcal{R}}(\hat{\pi})\bigr|\leq
\beta_1\cdot W_1(p_0^\dagger,p_0)+\beta_2\cdot W_1(p^\dagger, p)\\&
+4\sum_{h=1}^H\hat{\mathfrak{R}}_h(\mathcal{Q}_\Pi)+6H^2\sqrt{\frac{\log(4H/\delta)}{2T}}+
H(H+1)\varphi\\
&+2L_\Pi H \sqrt{2\log|\mathcal A|},
\end{align*}
where $\beta_1=2L_sH$ and $\beta_2=2H^2L_s(L_p+1)$.
\end{corollary}}
}

{
\subsection{Sim-to-Real Transfer Challenge}

Simulation is a common training ground for reinforcement learning, which often suffers from the \emph{reality gap}: a policy that performs excellently in a simulator can fail spectacularly in the real world, which differs in subtle but consequential ways. The mismatch can be reflected in distribution shifts in both observations and transition dynamics, and the learned policy may overfit to simulator-specific quirks rather than robust principles. This challenge is also referred to as the \emph{sim-to-real transfer} Challenge \citep{tobin2017domainrandomization, peng2018sim, andrychowicz2020learning}. 

Our theory provides a novel and powerful lens to study the sim-to-real transfer challenges. 
%The expected rational value risk of a learned policy $\pi$ refers to its deployment performance with transition kernel $p^\dagger$ and initial state distribution $p^\dagger_0$. 
Specifically, the extrinsic rational gap partially characterises this challenge and sheds light on how to mitigate it in terms of rationality. In the following section for experiments, we empirically study how environment shifts would negatively influence rationality. %, underpinned by between this policy's action and rational action $a^\circ$, when $\pi$ is learned to approximate $\pi^\circ$ from the training environment.

\subsection{Asymptotic Rationality}

%Combining Theorems~\ref{mthm:bias-unified}, and~\ref{mthm:multi-source-rl},
We also directly obtain the following corollary on the asymptotic property of rationality. %It follows the settings aforementioned. %asymptotic rationality result for the reinforcement learning agent in deployment.

\begin{corollary}[asymptotic rational risk gap bound]
\label{mcor:main_theory}
Under the same conditions of Theorems \ref{mthm:bias-unified} and \ref{mthm:multi-source-rl}, %Suppose the value function class $\mathcal{Q}_\Pi$ with a sequence of states $\mathbf s_h^{1:T}=\{s_h^t\}_{t=1}^T$ at time step $h\in [H]$ has empirical Rademacher complexity $\hat{\mathfrak{R}}_h(\mathcal{Q}_\Pi)$. 
for %any policy $\pi \in \Pi$ and 
any $\delta\in(0,1)$, with probability at least $1-\delta$, the rational risk gap of policy $\hat{\pi} \in \Pi$ over $T$ episodes of horizon $H$ can be bounded by:
\begin{align*}
&\lim_{T \to \infty} \left|\mathcal R(\hat{\pi})-\hat{\mathcal{R}}(\hat{\pi})\right|\leq
\beta_1\cdot W_1(p_0^\dagger,p_0)+\beta_2\cdot W_1(p^\dagger, p)\\&+2L_\Pi H \sqrt{2\log|\mathcal A|}
+4\sum_{h=1}^H\hat{\mathfrak{R}}_h(\mathcal{Q}_\Pi),
\end{align*}
where $\beta_1=2L_sH$ and $\beta_2=2H^2L_s(L_p+1)$. %; the $1$-Wasserstein distances $W_1(p^\dagger_0,p_0)$ and $W_1(p^\dagger,p)$, and Lipschitz constants $L_s$, $L_p$, and $L_\Pi$ are given above. %are given in Assumption~\ref{mass:lipschitz-val},~\ref{mass:lipschitz-env}, and \ref{mass:lipschitz-policy}. %is the Lipschitz constant of the mapping from a policy $\pi$ to its induced state distribution, and $|\mathcal A|$ denotes the cardinality of the action space.
\end{corollary}

{

}

\section{Experiments}

We conduct experiments to % measure the rational risk gap of reinforcement learning algorithms and 
empirically verify our measures and theoretical analysis. %The code is in the supplementary materials and will be released.

\subsection{Empirically Testable Hypotheses}

A good theory can explain and suggest empirically testable hypotheses \citep{lakatos1968criticism, popper2005logic}. Our theory leads to the following hypotheses.

\paragraph{H1: Benefits of Regularisations} \label{h1}
Regularisers, such as layer normalisation (LN) \cite{ba2016layer}, $\ell_2$ regularisation (L2), and weight normalisation (WN) \cite{salimans2016weight}, can penalise hypothesis complexity \cite{awasthi2020rad}. As suggested in Theorem~\ref{mthm:multi-source-rl}, %when the training time (i.e., the episode number $T$) increases, the dominant terms in rational risk gap bounds are dominated by the 
the reduced hypothesis complexity (measured here by the empirical Rademacher complexity), which indicates a smaller rational risk gap, corresponding to an improved rationality. % $\sum_{h=1}^H\hat{\mathfrak{R}}_h(\mathcal Q_\Pi)$.

%Motivated by these results, our first experiment tests the hypothesis that increased model complexity leads to large rational risk gap, and that penalisation techniques, such as regularisation, can effectively reduce this gap.

\paragraph{H2: Benefits of Domain Randomisation} \label{h2}
{{Domain randomisation (DR) is an augmentation technique that randomises parameters of the environment during training \citep{tobin2017domainrandomization}. It is supposed to improve the robustness of reinforcement learning algorithms against distribution shifts across environments. As suggested in Theorem~\ref{mthm:bias-unified}, this further improves the rationality. }

\begin{figure}[t]
    \centering
    \includegraphics[width=\linewidth]{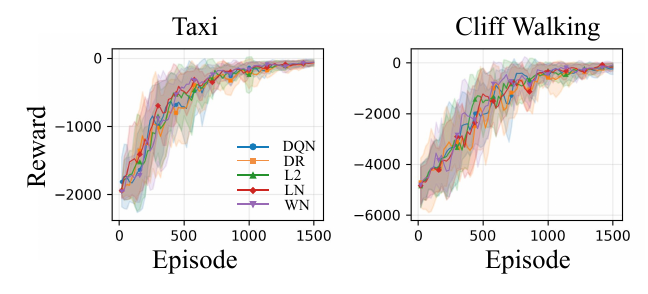}
    \caption{Reward curves of DQN under different regularisation and domain randomisation techniques in Taxi-v3 and Cliff Walking environments.}
    \label{mfig:reward}
\end{figure}

\paragraph{H3: Deficits of Environment Shifts} \label{h3}
Theorem~\ref{mthm:bias-unified} suggests that environment shifts enlarge the rational risk gap, %is influenced by the distribution shift between the training and inference environments, 
as quantified by the $1$-Wasserstein distance between transition kernels $W_1(p,p^\dagger)$ and initial state distributions $W_1(p_0,p_0^\dagger)$. Consequently, this means larger environment shifts lead to worse rationality. 

% In addition, we investigate the role of the learning rate, which is known to act as an implicit regulariser in reinforcement learning and affects optimisation stability \cite{amos2017,fengxiang2019}.

\subsection{Implementation Details}
\label{sec:imp}
We present major implementation details below. Full details are given in Appendix~\ref{app:A}. Code is available at \href{https://github.com/EVIEHub/Rationality}{https://github.com/EVIEHub/Rationality}.

\paragraph{Environment Setups}
Two popular Gym environments are employed in our experiments: Taxi-v3 \citep{Dietterich2000} and Cliff Walking \citep{sutton18}. We modify their environment dynamics to create two distinguished, training and inference settings. For the training environment, we choose the action randomisation rate from $0\%$ to $70\%$, whereby the environment may override the agent’s learned action with a random action. The inference environment takes the original environment without randomisation. Agents are trained under a non-zero probability of action randomisation, and then evaluated in the inference environment. In this way, we simulate distribution shifts between the training and inference environments.
% Table~\ref{tab:procgen_envs} summarises the environments and their observation and action space dimensions.

% \begin{figure*}[t!]
%   \centering
%   \includegraphics[width=\textwidth]{img/rationality-1.pdf}
%   \caption{Rational risk gap of DQN under different regularisation and domain randomisation techniques in Taxi-v3 and Cliff Walking environments. We compare vanilla DQN with Layer Normalisation (LN), $\ell_2$ Regularisation (L2), Weight Normalisation (WN), and Domain Randomisation (DR). The mean and standard deviation are shown across 5 independent runs.}
%   \label{fig:six_figures_reg}
% \end{figure*}

\begin{figure*}[t!]
  \centering
  \begin{subfigure}[t]{0.47\textwidth}
    \centering
    \includegraphics[width=\linewidth]{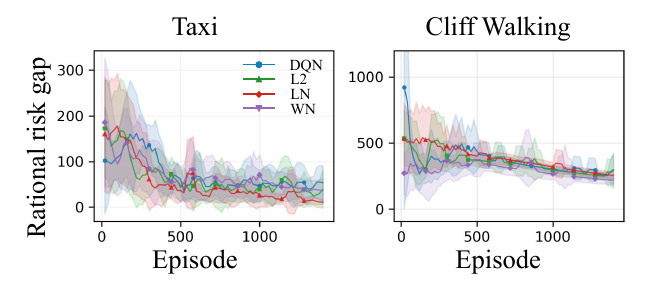}
    \caption{Regularisation}
    \label{mfig:reg}
  \end{subfigure}
  \hfill
  \begin{subfigure}[t]{0.47\textwidth}
    \centering
    \includegraphics[width=\linewidth]{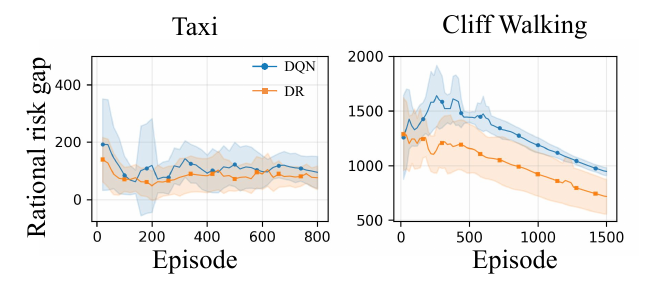}
    \caption{Domain randomisation}
    \label{mfig:dr}
  \end{subfigure}
  \caption{Rational risk gap of DQN under different regularisation and domain randomisation techniques in Taxi-v3 and Cliff Walking environments. }
  \label{mfig:six_figures_reg}
\end{figure*}

% \begin{figure*}[t!]
%   \centering
%   \includegraphics[width=\textwidth]{img/rationality_rl-2.pdf}
%   \caption{Rational risk gap of PPO under different data augmentation techniques across six Procgen environments.
% We compare vanilla PPO with random cropping (Crop) and cutout (Cutout). The mean and standard deviation are shown across 3 independent runs.}
%   \label{fig:six_figures_aug}
% \end{figure*}

\begin{figure}[t]
  \centering
  \includegraphics[width=\linewidth]{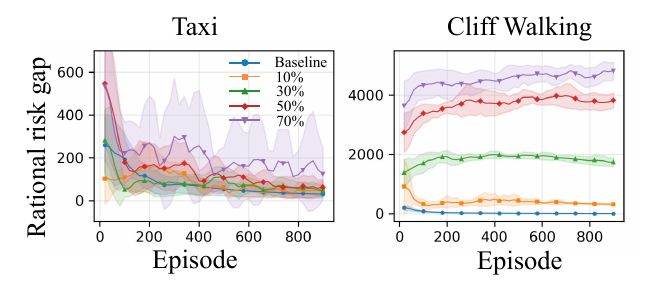}
  \caption{Rational risk gap of DQN across different environment levels in Taxi-v3 and Cliff Walking environments.}% We evaluate DQN under increasing challenge levels of training environments (0\%, 10\%, 30\%, 50\%, 70\%), presenting the probability of action randomisation during training. }
  \label{mfig:six_figures_level}
\end{figure}

\paragraph{Reinforcement Learning Algorithm} We employ a typical reinforcement learning algorithm, Deep Q-Network (DQN) \cite{mnih2013playing} with softmax action selection, in our experiments. 

%\paragraph{Regularisation} We consider several regularisation techniques in this paper, including layer normalisation \cite{ba2016layer}, $\ell_2$ regularisation, and weight normalisation \citep{salimans2016weight}. %, and domain randomisation \citep{wang2020improving, tobin2017domainrandomization}. 

\paragraph{Training Implementations} Agents are trained on a finite number of challenge levels, with the probabilities of executing a random action chosen from $\{0\%, 10\%, 30\%, 50\%, 70\%\}$. 
All results are averaged over five independent runs, with standard deviations reported as shaded regions.

{\paragraph{Experiment Design} For verifying Hypotheses {H1} and {H2}, agents are trained in both environments with challenge level of {10\%} and evaluated in the original environments. We repeat experiments with {five} random seeds and compare rational risk gaps across methods. For Hypothesis {H3}, we fix the DQN's hyperparameters and vary challenge levels in \{0\%, 10\%, 30\%, 50\%, 70\%\}, constructing different transition kernels. We repeat experiments and measure rational risk gaps in original environments. }

%\paragraph{Justification} Our experiment settings ensure that the actual value functions are accessible, enabling rigorous empirical verification of our rationality measures and theory. A more detailed justification is in Appendix \ref{app:B}.

%\paragraph{Reproducibility} The code is available at \href{https://github.com/EVIEHub/Rationality}{https://github.com/EVIEHub/Rationality}.

%\paragraph{Equipment} All experiments are conducted on three A100 GPUs (48GB memory). We use the OpenAI Gym library.

\subsection{Experimental Results}

{All setups run reasonably well in terms of reward, as shown in Figure \ref{mfig:reward}.} This ensures that our experiments are for rationality, controlling irrelevant variables.

\paragraph{H1: Regularisation}
Figure~\ref{mfig:reg} illustrates the benefits of regularisation on DQN across the considered environments. In both environments, %rational risk gap decreases monotonically with the number of episodes: 
$\ell_2$ regularisation consistently reduces rational risk gap; layer normalisation provides a stronger control in Taxi-v3 environment; and weight normalisation is more effective in the Cliff Walking environment compared to vanilla DQN. {Table~\ref{mtab:results} indicates $10^{-3}$ is the most suitable $\ell_2$ regularisation strength for minimising the rational risk gap in both environments. }

% \bry{
% \begin{table}[h]
% \caption{Your caption here.}
% \centering
% \begin{tabular}{lcc}
% \toprule
% \textbf{Variable} & \textbf{Taxi} & \textbf{Cliff Walking} \\
% \midrule
% DQN       & $52.6 \pm 13.2$ & $934.9 \pm 934.9$ \\
% l2\_1e-3  & $36.9 \pm 1.10$ & $863.8 \pm 56.9$  \\
% l2\_1e-5  & $47.2 \pm 7.8$  & $855.5 \pm 72.8$  \\
% l2\_1e-7  & $42.9 \pm 10.3$ & $878.2 \pm 14.9$  \\
% \bottomrule
% \end{tabular}
% \label{tab:results}
% \end{table}
% }

{
\begin{table}[h]
\caption{Rational risk gap of DQN across different $\ell_2$ regularisation strength in Taxi-v3 and Cliff Walking environments.}% We evaluate DQN under increasing regularisation strengths ($10^{-3}$,$10^{-4}$,$10^{-5}$,$10^{-6}$,$10^{-7}$).}
\centering
\small
\begin{tabular}{lcc}
\toprule
\textbf{Variable} & \textbf{Taxi} & \textbf{Cliff Walking} \\
\midrule
DQN       & $35.34 \pm 22.91$ & $206.67 \pm 26.50$ \\
$10^{-3}$  & $\bm{15.07 \pm 5.71}$ & $\bm{150.13 \pm 18.08}$  \\
$10^{-4}$  & $17.24 \pm 5.66$ & $204.48 \pm 24.52$  \\
$10^{-5}$  & $26.95 \pm 25.04$ & $167.81 \pm 12.83$  \\
$10^{-6}$  & $19.31 \pm 6.79$  & $206.66 \pm 46.32$  \\
$10^{-7}$  & $16.16 \pm 5.55$ & $162.72 \pm 16.84$  \\
\bottomrule
\end{tabular}
\label{mtab:results}
\end{table}
}
%but their effectiveness is limited across different environments.
%These results support our theoretical analysis, indicating that regularisation-based control of value function class complexity effectively mitigates the rational risk gap over training.
\paragraph{H2: Domain Randomisation}
Figure~\ref{mfig:dr} illustrates the benefits of domain randomisation on the rationality. {Compared to the DQN baseline, domain randomisation effectively reduces the rational risk gap in both environments, especially in the Cliff Walking environment.
% cropping consistently achieves the largest reduction in rational risk gap. In particular, in Fruitbot and Starpilot, cropping reduces the final rational risk gap by a large margin.
%These results are consistent with our theoretical analysis of environment shift. Instead of increasing the coverage of training level distributions, data augmentation effectively expands the support set of the training distribution to produce transition kernels that approximate those of the inference environment. As a result, the discrepancy between the induced transition kernel and the inference kernel, $W_1(p,p^\dagger)$, is reduced.

\paragraph{H3: Environment Shifts}
Figure~\ref{mfig:six_figures_level} reports the rational risk gap of DQN under different challenge levels of training environments. Rational risk gap shows a clear, positive correlation with the challenge levels, which fully supports the hypothesis that environment shifts are harmful to rationality. %With the increase of training episode, the effect of environment shifts cannot be mitigated, supporting Lemma~\ref{mthm:bias-unified}. {However, the effect of changes in the initial state distribution $W_1(p_0^\dagger, p_0)$ on the rational risk gap is mild.}}

{

}

\section{Conclusions and Future Works}

We introduce a rationality framework for reinforcement learning agents, an understudied but increasingly important lens for interpreting AI behaviour. We mathematically define %call an action 
perfectly rational actions, %at deployment if it maximises the (unobserved) true value function, 
and quantify bounded rationality by a rational risk gap. %departures from this ideal via expected rational risk, defined as the cumulative value discrepancy between a policy’s chosen actions and their rational counterparts; we also define an empirical training analogue. 
The rational risk gap admits a clean decomposition into an extrinsic component %component driven by train–test environment shift 
, and an intrinsic one, %attributable to the learning algorithm itself. We further show that these terms are 
each controlled by an upper bound. %, respectively, the $1$-Wasserstein distance between training and deployment transition dynamics and initial-state distributions, and the value-function class’s empirical Rademacher complexity. 
These bounds yield concrete practical implications: regularisation %(e.g., batch normalisation, $\ell_2$ regularisation, network randomisation, randomized convolution) 
and domain randomisation can reduce intrinsic irrationality, while environment shift predictably worsens extrinsic irrationality. Comprehensive experiments %across six Procgen environments 
support these predictions, collectively validating our theory. 
%\paragraph{Future works:}
We will develop economic analysis from the following directions:
\paragraph{(In-)Stability of Multi-Agent System} Our rationality measures quantify how far an agent’s behaviour could deviate from its optimal strategy, which can further help understand how the dynamics of a multi-agent system would deviate from the “ideal” equilibrium. This offers a lens for understanding the (in-)stability of this multi-agent system. More technically, a direct intuition is: an agent’s action can be modelled as the (perfectly) rational actions plus a (random) deviation bounded by the rationality measures. This deviation will make the dynamics away from the ideal equilibrium.

% \paragraph{Mechanism Design in an Irrational System} Our theory offers a tool for characterising a society / system that is not perfectly rational; in particular, how imperfection would be. This enables mechanism design for incentivising certain behaviour in this imperfect, and more realistic, setting. This is largely absent in the literature.

% \paragraph{Economic Simulation} \fh{consider merging this and the previous one. these two are very good directions. let's discuss how to proceed if you are interested} \bry{sure, I am also thinking about this direction, it could be suitable for my next research to do the social / economic environment generation rather than physical environment generation} Our theory can be the foundation for more realistic modelling of an economic system with imperfectly rational agents. It allows simulating the impact of different levels of rationality on collective economic outcomes, such as market efficiency, wealth distribution, and the discrepancy with the predictions of fully rational models.

% {\paragraph{Mechanism Design and Economic Simulation}
% Our theory provides a tool for modelling economic systems with imperfectly rational agents. This enables mechanism design under imperfect, and more realistic, behavioural assumptions and supports simulations of the impact of different rationality levels on collective outcomes, such as market efficiency, wealth distribution, and the discrepancy with the predictions of fully rational models.}

\paragraph{Socioeconomic Simulation and Mechanism Design}
Our model and theory inform better simulation of socioeconomic systems with imperfectly rational agents, such as financial markets and urban interactions. The simulation further supports mechanism design under more realistic behavioural assumptions, where incentives can be designed for agents that are not perfectly rational, for example.

\paragraph{Multi-Agent System Design}
Further, as a new tool for quantifying and calibrating action deviation from the perfectly rational actions, our rationality measures enable a systemic approach to managing irrationality in orchestrating a multi-agent system. Intuitively, simple examples are (1) the system-level deviation can be mitigated by applying parallel circuits, and (2) we can better “allocate resources” in improving system performance of a series circuit by prioritising the poorliest performing agent wherein. This systematic view helps strike a good balance between cost and system-level performance in agent orchestration. Through this, one could also advance fields like survival analysis and complex systems.

\section*{Acknowledgements}

KQ was supported in part by the UKRI Grant EP/Y03516X/1 for the UKRI Centre for Doctoral Training in Machine Learning Systems (\href{https://mlsystems.uk/}{https://mlsystems.uk/}).

\section*{Impact Statement}

% The goal of this paper is to advance the field of Machine
% Learning. There are many potential societal consequences, none 
% of which we feel must be specifically highlighted here.

This paper presents work whose goal is to advance the field of machine learning. There are many potential societal consequences of our work, none of which we feel must be specifically highlighted here.


\bibliography{ref}
\bibliographystyle{icml2026}

%%%%%%%%%%%%%%%%%%%%%%%%%%%%%%%%%%%%%%%%%%%%%%%%%%%%%%%%%%%%%%%%%%%%%%%%%%%%%%%
%%%%%%%%%%%%%%%%%%%%%%%%%%%%%%%%%%%%%%%%%%%%%%%%%%%%%%%%%%%%%%%%%%%%%%%%%%%%%%%
% APPENDIX
%%%%%%%%%%%%%%%%%%%%%%%%%%%%%%%%%%%%%%%%%%%%%%%%%%%%%%%%%%%%%%%%%%%%%%%%%%%%%%%
%%%%%%%%%%%%%%%%%%%%%%%%%%%%%%%%%%%%%%%%%%%%%%%%%%%%%%%%%%%%%%%%%%%%%%%%%%%%%%%
\newpage
\appendix
\onecolumn

\section{Notation}

\begin{table}[h!]
\centering
\renewcommand{\arraystretch}{1.2}
\caption{Notation}
\begin{tabularx}{0.71\textwidth}{@{} c X @{}}
\toprule
\multicolumn{1}{c}{\textbf{Symbol}} & \multicolumn{1}{c}{\textbf{Description}} \\
\midrule

$\mathcal{S}$ & State space \\
$\mathcal{A}$ & Finite action space \\
$|\mathcal{A}|$ & Cardinality of action space \\
$H$ & Horizon length  \\
$T$ & Number of training episodes \\

$s^t_h$ & State at step $h$ of episode $t$ \\

$\pi$, $\hat{\pi}$ & Stochastic policy, mapping states to action distributions \\
% $\pi$ & Stochastic policy, mapping states to action distributions \\
$\pi^\ast$ & optimal policy under initial state distribution $p_0^\dagger$ and transition kernel $p^\dagger$ in deployment\\

$p$ & Transition kernel of the training environment \\
$p^\dagger$ & Transition kernel of the inference environment \\

$p_0$ & Initial state distribution of the training environment \\
$p_0^\dagger$ & Initial state distribution of the inference environment \\

$\mathcal{D}^{\pi}_{h}$ & State distribution at step $h$ induced by policy $\pi$ under $p$ \\
$\mathcal{D}^{\pi,\dagger}_{h}$ & State distribution at step $h$ induced by policy $\pi$ under $p^\dagger$ \\

$r_h$ & Reward function at step $h$ \\
$V_h^{\pi}(s)$ & Value function of policy $\pi$ at step $h$ under transition kernel $p$\\
$Q_h^{\pi}(s,a)$ & Action value function of policy $\pi$ at step $h$ transition kernel $p$\\

$V_h^{\pi,\dagger}(s)$ & Value function of policy $\pi$ at step $h$ under transition kernel $p^\dagger$ \\
$Q_h^{\pi,\dagger}(s,a)$ & Action value function of policy $\pi$ at step $h$ under transition kernel $p^\dagger$\\

$\mathcal R_h(\pi)$ & expected rational value loss of policy $\pi$ at step $h$ \\

$\hat {\mathcal R}_h(\pi)$ & empirical rational value loss of policy $\pi$ at step $h$ \\
$\mathcal R(\pi)$ & expected rational value risk of policy $\pi$\\

$\hat {\mathcal R}(\pi)$ & empirical rational value risk of policy $\pi$\\

$\mathcal{Q}_\Pi$ & Class of value functions $\mathcal{Q}_\Pi:\mathcal S \to \mathbb R$ \\

$\mathfrak R(\mathcal{Q}_\Pi)$ & Rademacher complexity of $\mathcal{Q}_\Pi$ \\
$\hat{\mathfrak R}(\mathcal{Q}_\Pi)$ & Empirical Rademacher complexity of $\mathcal{Q}_\Pi$ \\
$W_1(p^\dagger_0,p_0)$ & $1$-Wasserstein distance between initial state distributions of $p^\dagger_0$ and $p_0$ \\
$W_1(p^\dagger,p)$ & $1$-Wasserstein distance between state distributions induced by $p^\dagger$ and $p$, i.e., $W_1(p^\dagger,p)\triangleq\sup_{s\in \mathcal{A},a\in\mathcal{A}}W_1(p^\dagger(\cdot|s,a),p(\cdot|s,a))$ \\
$d_\Pi(\pi,\pi')$ & TV distance between policies $\pi$ and $\pi'$ for any $\pi,\pi' \in \Pi$ and $s\in\mathcal{S}$, i.e., $d_\Pi(\pi,\pi')\triangleq\sup_s d_\Pi(\pi(s),\pi'(s))$\\
$L_s$ & Lipschitz constant of value functions w.r.t. states \\
$L_p$ & Lipschitz constant of induced state distributions w.r.t. transition kernels \\
$L_\Pi$ & Lipschitz constant of induced state distributions w.r.t. policy\\

$\mathcal O(\cdot)$ & Asymptotic complexity notation \\
% $\tilde{\mathcal O}(\cdot)$ & Asymptotic complexity, suppressing logarithmic factors \\

\bottomrule
\end{tabularx}
\end{table}

\section{Proof of Theorem~\ref{mthm:bias-unified}}
\label{addsec:ege}

In this section, we prove the extrinsic rational risk bound in Theorem~\ref{mthm:bias-unified}. We define the integral probability metric (IPM).

\begin{definition}[IPM]
\label{def:ipm}
Let $\mathcal{Q}_\Pi\subseteq\{f:\mathcal{S}\to\mathbb{R}\}$ be a class of bounded measurable functions. For any probability measures $\mu, \nu$ on $\mathcal{S}$, the integral probability metric (IPM) induced by $\mathcal{Q}_\Pi$ is defined as
\[
    D_{\mathcal{Q}_\Pi}(\mu,\nu)
    \triangleq
    \sup_{f\in\mathcal{Q}_\Pi}
    \bigl|
        \mathbb{E}_{\mu}[f]
        -
        \mathbb{E}_{\nu}[f]
    \bigr|.
\]
\end{definition}

\begin{comment}
\begin{assumption}[lipschitz dependence on policy]
  \label{ass:lipschitz-policy}
  Let $(\Pi,d_\Pi)$ be a metric space. Assume that the map $\pi \mapsto \mathcal{D}_{h}^{\pi}$ is
  $L_\Pi$-Lipschitz with respect to the IPM $D_{\mathcal{Q}_\Pi}$, i.e.,
  \[
    D_{\mathcal{Q}_\Pi}(\mathcal{D}^{\pi}_h,\mathcal{D}^{\pi'}_h)
    \le
    L_\Pi\, d_\Pi(\pi,\pi'),
    \quad
    \forall \pi,\pi'\in\Pi.
  \]
\end{assumption}

\begin{assumption}[episode-level independence]
\label{ass:episode-independence}
For each $t=1,\dots,T$, the state $s^t_h$ is generated by running policy
$\pi_t\in\Pi$ in environment $S$,
\(
    s^t_h \sim \mathcal{D}_{h}^{\pi_t}.
\)
The variables $\{s^t_h\}_{t=1}^T$ are independent, though not necessarily
identically distributed.
\end{assumption}
\end{comment}

We now restate our Lemma~\ref{mlem:trans_error}.

\label{app:them1}
\renewcommand{\thelemma}{3}
\begin{lemma}
\label{lem:trans_error}
Under Assumption \ref{mass:lipschitz-val}, for any step $h\in[H]$, the optimal value discrepancy between the inference transition kernel $p^\dagger$ and training transition kernel $p$ under same training distribution $\mathcal D_h^{\hat {\pi}}$ satisfies
\begin{align*}
\sup_{\pi \in \Pi} \left|
\mathbb E_{s_h\sim\mathcal D_h^{\hat{\pi}}}
Q_h^{*,\dagger}(s_h,a_h^\pi)
-
\mathbb E_{s_h\sim \mathcal D_h^{\hat{\pi}}}
Q_h^{*}(s_h,a_h^{\pi})
\right|\le
(H-h)\,L_s\,W_1(p^\dagger,p).
\end{align*}
\end{lemma}
\renewcommand{\thelemma}{\arabic{lemma}}

\begin{proof}
For any $h\in\{1,\dots,H\}$, the Bellman expectation equations give
\[
Q_h^{*}(s_h,a_h^\pi)
=
r_h(s_h,a_h^\pi)
+
\int_{\mathcal S} V_{h+1}^{*}(s_{h+1})\,p(\mathrm ds_{h+1}\mid s_h,a_h^\pi),
\]
and
\[
Q_h^{*,\dagger}(s_h,a_h^\pi)
=
r_h(s_h,a_h^\pi)
+
\int_{\mathcal S} V_{h+1}^{*,\dagger}(s_{h+1})\,p^\dagger(\mathrm ds_{h+1}\mid s_h,a_h^\pi).
\]

Subtracting the two equations yields
\begin{align*}
Q_h^{*,\dagger}(s_h,a^\pi_h)-Q_h^{*}(s_h,a^\pi_h) =
\int_{\mathcal S} V_{h+1}^{*,\dagger}(s_{h+1})\,p^\dagger(\mathrm ds_{h+1}\mid s_h,a_h^\pi)
-
\int_{\mathcal S} V_{h+1}^{*}(s_{h+1})\,p(\mathrm ds_{h+1}\mid s_h,a_h^\pi)
.
\end{align*}

Adding and subtracting
$\int_{\mathcal S} V_{h+1}^{*}(s_{h+1})\,p^\dagger(\mathrm ds_{h+1}\mid s_h,a_h^\pi)$
inside the integrand gives
\begin{align*}
&Q_h^{*,\dagger}(s_h,a_h^\pi)-Q_h^{*}(s_h,a_h^\pi) \\
&=
\int_{\mathcal S}
\bigl(
V_{h+1}^{*,\dagger}(s_{h+1})-V_{h+1}^{*}(s_{h+1})
\bigr)
p^\dagger(\mathrm ds_{h+1}\mid s_h,a_h^{\pi})
+
\int_{\mathcal S}
V_{h+1}^{*}(s_{h+1})\,
\bigl(p^\dagger-p\bigr)(\mathrm ds_{h+1}\mid s_h,a_h^\pi).
\end{align*}

According to Assumption~\ref{mass:lipschitz-val}, $V_{h+1}^{*}(\cdot)$ is $L_s$-Lipschitz and by the
Kantorovich--Rubinstein duality, we have
\[\left|
\int_{\mathcal S}
V_{h+1}^{*}(s_{h+1})\,
\bigl(p^\dagger-p\bigr)(\mathrm ds_{h+1}\mid s_h,a_h^\pi)
\right|
\le
L_s\,W_1(p^\dagger(\cdot|s_h,a_h^\pi),p(\cdot|s_h,a_h^\pi)).
\]

Taking absolute values and using the triangle inequality,
\begin{align*}
\bigl|
Q_h^{*,\dagger}(s_h,a_h^\pi)-Q_h^{*}(s_h,a_h^\pi)
\bigr|
&\le
\left|\int_{\mathcal S}\bigl(
V_{h+1}^{*,\dagger}(s_{h+1})-V_{h+1}^{*}(s_{h+1})\bigr)
p^\dagger(\mathrm ds_{h+1}\mid s_h,a_h^\pi)\right| 
+ L_s\,W_1(p^\dagger(\cdot|s_h,a_h^\pi),p(\cdot|s_h,a_h^\pi)).
\end{align*}

Hence, for any $h\in\{1,\dots,H\}$:
\begin{align}
\label{eq:recursion}
\sup_{s\in\mathcal S,a\in\mathcal A}
\bigl|
Q_h^{*,\dagger}(s,a)-Q_h^{*}(s,a)
\bigr|
&\le
\sup_{s\in\mathcal S,a\in\mathcal A}
\left|
\int_{\mathcal S}
\bigl(
V_{h+1}^{*,\dagger}(s')-V_{h+1}^{*}(s')
\bigr)
p^\dagger(\mathrm ds'\mid s,a)
\right|\\
&\quad+\sup_{s\in\mathcal S,a\in\mathcal A}L_sW_1(p^\dagger(\cdot\mid s,a),p(\cdot\mid s,a))
\nonumber\\
&\le
\sup_{s'\in\mathcal S}
\bigl|
V_{h+1}^{*,\dagger}(s')-V_{h+1}^{*}(s')
\bigr|
+L_sW_1(p^\dagger,p)
\nonumber\\
&\le
\sup_{s'\in\mathcal S,a'\in\mathcal{A}}
\bigl|
Q_{h+1}^{*,\dagger}(s',a')-Q_{h+1}^{*}(s',a')
\bigr|+ L_s\,W_1(p^\dagger,p).
\end{align}

% \begin{align}
% \label{eq:recursion}
% \bigl|
% Q_h^{*,\dagger}(s,a)-Q_h^{*}(s,a)
% \bigr|
% &\le
% \left|
% \int_{\mathcal S}
% \bigl(
% V_{h+1}^{*,\dagger}(s')-V_{h+1}^{*}(s')
% \bigr)
% p^\dagger(\mathrm ds'\mid s,a)
% \right|
% +L_sW_1(p^\dagger,p)
% \nonumber\\
% &\le
% \bigl|
% V_{h+1}^{*,\dagger}(s)-V_{h+1}^{*}(s)
% \bigr|
% +L_sW_1(p^\dagger,p)
% \nonumber\\
% &\le
% \bigl|
% Q_{h+1}^{*,\dagger}(s,a)-Q_{h+1}^{*}(s,a)
% \bigr|+ L_s\,W_1(p^\dagger,p).
% \end{align}

We prove by backward induction on $h$, for all $h\in\{1,\dots,H\}$,
\begin{equation}\label{eq:induction-claim}
\sup_{s\in\mathcal S,a\in\mathcal A}\bigl|
Q_h^{*,\dagger}(s,a)-Q_h^{*}(s,a)
\bigr|
\le
(H-h)\,L_s\,W_1(p^\dagger,p).
\end{equation}
  
By the terminal condition $V_{H+1}^\pi(\cdot)\equiv V_{H+1}^{\pi,\dagger}(\cdot)\equiv 0$, we have
\begin{align*}
\sup_{s\in\mathcal S,a\in\mathcal A}\bigl|
Q_H^{*,\dagger}(s,a)-Q_H^{*}(s,a)
\bigr|=\sup_{s\in\mathcal S,a\in\mathcal A}
\bigl|
\int_{\mathcal{S}} V_{H+1}^{*,\dagger}(s')p^\dagger(ds'|s,a)-\int_{\mathcal S} V_{H+1}^{*}(s')p(ds'|s,a)
\bigr|=0.
\end{align*}
Moreover, note that the Right-hand Side (RHS) of equation \eqref{eq:induction-claim} at $h=H$ equals
\[
(H-H)\,L_s\,W_1(p^\dagger,p)=0,
\]
so equation \eqref{eq:induction-claim} holds for $h=H$.

Fix any $h\in\{1,\dots,H-1\}$. Assume that equation \eqref{eq:induction-claim} holds at time $h+1$, we have
\begin{align}\label{eq:IH}
\sup_{s\in\mathcal S,a\in\mathcal A}\bigl|
Q_{h+1}^{*,\dagger}(s,a)-Q_{h+1}^{*}(s,a)
\bigr|
&\le
(H-h-1)\,L_s\,W_1(p^\dagger,p).
\end{align}
Applying the recursion equation \eqref{eq:recursion} and then substituting equation \eqref{eq:IH}, we obtain
\begin{align*}
\sup_{s\in\mathcal S,a\in\mathcal A}\bigl|
Q_{h}^{*,\dagger}(s,a)-Q_{h}^{*}(s,a)
\bigr|
&\le
\sup_{s\in\mathcal S,a\in\mathcal{A}}
\bigl|
Q_{h+1}^{*,\dagger}(s,a)-Q_{h+1}^{*}(s,a)
\bigr|+ L_s\,W_1(p^\dagger,p) \\
&\le
(H-h-1)\,L_s\,W_1(p^\dagger,p)
+ L_s\,W_1(p^\dagger,p) \\
&=
(H-h)\,L_s\,W_1(p^\dagger,p).
\end{align*}
This proves that equation \eqref{eq:induction-claim} holds at time $h$ whenever it holds at time $h+1$.

By backward induction from $h=H$ down to $h=1$, equation \eqref{eq:induction-claim} holds for all $h\in\{1,\dots,H\}$.

The final claim follows since
\begin{align*}
\sup_{\pi \in \Pi} \left|
\mathbb E_{s_h\sim\mathcal D_h^{\hat{\pi}}}
Q_h^{*,\dagger}(s_h,a_h^\pi)
-
\mathbb E_{s_h\sim \mathcal D_h^{\hat{\pi}}}
Q_h^{*}(s_h,a_h^{\pi})
\right|&\le \sup_{s\in\mathcal S,a\in \mathcal{A}}
\bigl|
Q_{h}^{*,\dagger}(s,a)-Q_{h}^{*}(s,a)
\bigr|
\\
&\le
(H-h)\,L_s\,W_1(p^\dagger,p).
 \end{align*}
\end{proof}

We are ready to prove the upper bound on the extrinsic rational gap in Theorem~\ref{thm:bias-unified}.

\renewcommand{\thetheorem}{1}
\begin{theorem}[extrinsic rational gap bound]
\label{thm:bias-unified}
Let $\mathcal D_h^{\hat{\pi},\dagger}, \mathcal D_h^{\hat{\pi}}$ denote the state distributions in inference and training. Under Assumptions~\ref{mass:lipschitz-val}--\ref{mass:lipschitz-env}, the extrinsic rational gap over a trajectory of horizon $H$ is  upper bounded by
\begin{align*}
\sum_{h=1}^H \sup_{\pi \in \Pi} \left|\mathbb{E}_{s_h\sim\mathcal D_h^{\hat{\pi},\dagger}} Q_h^{*,\dagger}(s_h,a_h^{\pi})
-
 \mathbb{E}_{s_h \sim \mathcal D_h^{\hat{\pi}}}Q_h^{*}(s_h,a_h^\pi)\right|  \leq L_s H\cdot W_1(p_0^\dagger,p_0)+H^2L_s(L_p+1)\cdot W_1(p^\dagger, p).
\end{align*}
\end{theorem}
\renewcommand{\thetheorem}{\arabic{theorem}}

\begin{proof}
For each $h$, we add and subtract
\(
\mathbb E_{s_h\sim\mathcal D_h^{\hat\pi}}
Q_h^{*,\dagger}(s_h,a_h^\pi)
\)
and then apply the triangle inequality and take the supremum over $\pi\in\Pi$.

\begin{align*}
&\sup_{\pi \in \Pi} \left|\mathbb{E}_{s_h\sim\mathcal D_h^{\hat{\pi},\dagger}} Q_h^{*,\dagger}(s_h,a_h^\pi)
-
\mathbb E_{s_h \sim \mathcal D_h^{\hat{\pi}}} Q_h^{*}(s_h,a_h^\pi)\right| \\
&
 {\leq} 
\underbrace{
\sup_{\pi \in \Pi} \left|\mathbb{E}_{s_h\sim\mathcal D_h^{\hat{\pi},\dagger}} Q_h^{*,\dagger}(s_h,a_h^\pi)
-
\mathbb{E}_{s_h\sim\mathcal D_h^{\hat{\pi}}} Q_h^{*,\dagger}(s_h,a_h^\pi)\right|}_{\text{I}} \\
&\quad+ 
\underbrace{
\sup_{\pi \in \Pi} \left|\mathbb{E}_{s_h\sim\mathcal D_h^{\hat{\pi}}} Q_h^{*,\dagger}(s_h,a_h^\pi)
-
\mathbb{E}_{s_h\sim \mathcal D_h^{\hat{\pi}}} Q_h^{*}(s_h,a_h^\pi)\right|}
_{\text{II}}.
\end{align*}

\medskip
\noindent\textbf{Term I.} This term describes the discrepancy of distributions induced by the difference between two different transition kernels $p^\dagger$ and $p$ as well as the initial state distributions $p^\dagger_0$ and $p_0$ for any $s \in \mathcal{S}$ and $a \in \mathcal{A}$. 

According to the definition of IPM, the first term \(\sup_{\pi \in \Pi} \left|\mathbb{E}_{s_h\sim\mathcal D_h^{\hat{\pi},\dagger}} Q_h^{*,\dagger}(s_h,a_h^\pi)
-
\mathbb{E}_{s_h\sim\mathcal D_h^{\hat{\pi}}} Q_h^{*,\dagger}(s_h,a_h^\pi)\right| \)
satisfies:
\[
\sup_{\pi \in \Pi} \left|\mathbb{E}_{s_h\sim\mathcal D_h^{\hat{\pi},\dagger}} Q_h^{*,\dagger}(s_h,a_h^\pi)
-
\mathbb{E}_{s_h\sim\mathcal D_h^{\hat{\pi}}} Q_h^{*,\dagger}(s_h,a_h^\pi)\right|\leq D_{\mathcal Q_\Pi}(\mathcal{D}^{\hat{\pi},\dagger}_h,\mathcal{D}^{\hat{\pi}}_h).
\]
To relate this to the difference between kernels $p^\dagger$ and $p$, we use Assumption~\ref{mass:lipschitz-val}, which assumes that every $f\in\mathcal{Q}_\Pi$ is $L_s$-Lipschitz. By the Kantorovich--Rubinstein duality, this obtains
\[
    D_{\mathcal Q_\Pi}(\mathcal{D}^{\hat{\pi},\dagger}_h,\mathcal{D}^{\hat{\pi}}_h)
    \le
    L_s\, W_1(\mathcal{D}^{\hat{\pi},\dagger}_h,\mathcal{D}^{\hat{\pi}}_h).
\]

To bound the distribution shift by the 1-Wasserstein distance of initial state distributions and transition kernels, we first claim the 1-Wasserstein distance between state distribution $\mathcal{D}^{\hat{\pi},\dagger}_h$ and $\mathcal{D}^{\hat{\pi}}_h$ can be bounded by: $    W_1(\mathcal{D}^{\hat{\pi},\dagger}_h,\mathcal{D}^{\hat{\pi}}_h)\le
W_1(p_0^\dagger,p_0)+
    (h-1)L_{p}\, W_1(p^\dagger,p).
$ We prove this claim by induction on $h$.

For the base case $h=1$, we have $
W_1\left(\mathcal D_1^{\hat{\pi},\dagger},\mathcal D_1^{\hat{\pi}}\right)
=
W_1\left(p_0^\dagger,p_0\right).
$
Since
$
W_1\left(p_0^\dagger,p_0\right)
+
(1-1)L_p\,W_1\left(p^\dagger,p\right)
=
W_1\left(p_0^\dagger,p_0\right),
$
the claimed bound holds for $h=1$.

For some $h\in[H-1]$ in Assumption~\ref{mass:lipschitz-env}, we have
\(
W_1\left(\mathcal D_{h+1}^{\hat{\pi},\dagger},\mathcal D_{h+1}^{\hat{\pi}}\right)
\le
W_1\left(\mathcal D_h^{\hat{\pi},\dagger},\mathcal D_h^{\hat{\pi}}\right)
+
L_p\,W_1\left(p^\dagger,p\right).
\)

Plugging the induction hypothesis into the above inequality obtains
\begin{align*}
W_1\left(\mathcal D_{h+1}^{\hat{\pi},\dagger},\mathcal D_{h+1}^{\hat{\pi}}\right)
&\le
W_1\left(p_0^\dagger,p_0\right)
+
(h-1)L_p\,W_1\left(p^\dagger,p\right)
+
L_p\,W_1\left(p^\dagger,p\right) \\
&=
W_1\left(p_0^\dagger,p_0\right)
+
hL_p\,W_1\left(p^\dagger,p\right).
\end{align*}
Therefore, the bound also holds for $h+1$. By induction, for all $h\in[H]$,
\[
W_1\left(\mathcal D_h^{\hat{\pi},\dagger},\mathcal D_h^{\hat{\pi}}\right)
\le
W_1\left(p_0^\dagger,p_0\right)
+
(h-1)L_p\,W_1\left(p^\dagger,p\right).
\]
This completes the induction proof. Thus, the environment shift is bounded by:
\begin{equation}
\label{eq:es}
\sup_{\pi \in \Pi} \left|\mathbb{E}_{s_h\sim\mathcal D_h^{\hat{\pi},\dagger}} Q_h^{*,\dagger}(s_h,a_h^\pi)
-
\mathbb{E}_{s_h\sim\mathcal D_h^{\hat{\pi}}} Q_h^{*,\dagger}(s_h,a_h^\pi)\right|\leq L_s W_1(p_0^\dagger,p_0)+
    (h-1)L_sL_{p} W_1(p^\dagger,p).
\end{equation}

\medskip
\noindent\textbf{Term II.} This term quantifies the shift introduced by the difference between the transition kernel $p^\dagger$ in deployment and the transition kernel $p$ in training.
Based on the lemma~\ref{mlem:trans_error}, we have 
\begin{align}
\label{eq:te}
\sup_{\pi \in \Pi}\left|
\mathbb E_{s_h\sim\mathcal D_h^{\hat{\pi}}}
Q_h^{*,\dagger}(s_h,a_h^\pi)
-
\mathbb E_{s_h\sim\mathcal D_h^{\hat{\pi}}}
Q_h^{*}(s_h,a_h^\pi)
\right|
&\le
(H-h)L_sW_1(p^\dagger,p).
\end{align}

Combining these two bounds of \ref{eq:es} and \ref{eq:te},
\begin{align*}
&\sum_{h=1}^H \sup_{\pi \in \Pi} \left|\mathbb{E}_{s_h\sim\mathcal D_h^{\hat{\pi},\dagger}} Q_h^{*,\dagger}(s_h,a_h^{\pi})
-
 \mathbb{E}_{s_h \sim \mathcal D_h^{\hat{\pi}}}Q_h^{*}(s_h,a_h^\pi)\right| \\
&\leq \sum_{h=1}^H\left[L_s \cdot W_1(p_0^\dagger,p_0)+
    (h-1)L_sL_{p}\cdot W_1(p^\dagger,p)+(H-h)L_s\cdot W_1(p^\dagger,p)\right]\\
&\leq L_s H\cdot W_1(p_0^\dagger,p_0)+H^2L_s(L_p+1)\cdot W_1(p^\dagger, p),
\end{align*}
which concludes the proof.
\end{proof}

\section{Proof of Theorem~\ref{mthm:multi-source-rl}}
\label{addsec:ige}
In this section, we prove the upper bound on the intrinsic rational gap in Theorem~\ref{mthm:multi-source-rl}.

\begin{lemma}
\label{lem:policy_path_length}
Under Assumption~\ref{mass:klpolicy}, let $\mathcal A$ be a finite action space. Assume $\pi_1(\cdot\mid s)$ is uniform over $\mathcal A$ for all $s \in \mathcal S$, and for some $\alpha > 0$, $\sup_{s\in\mathcal S}\mathrm{KL}\!\bigl(\pi_{t+1}(\cdot\mid s)\,\|\,\pi_t(\cdot\mid s)\bigr)
\le \alpha,\quad \forall\, t=1,\dots,T-1.$ Then for all $t\ge 1$,
\begin{align*}
d_\Pi(\hat{\pi},\pi_t) \le \sqrt{\frac{\log|\mathcal A|}{2}}+\sqrt{\frac{(t-1)^2 \alpha}{2}}.
\end{align*}
\end{lemma}

\begin{proof}
By definition, for any $s\in\mathcal S$, the total variation distance
$d_{\Pi}(\cdot,\cdot)$ is a metric on the probability simplex over
$\mathcal A$, and satisfies the triangle inequality. Therefore, for any
$s\in\mathcal S$,
\[
d_{\Pi}\!\bigl(\hat{\pi}(\cdot\mid s),\pi_t(\cdot\mid s)\bigr)
\le
d_{\Pi}\!\bigl(\hat{\pi}(\cdot\mid s),\pi_1(\cdot\mid s)\bigr)
+
\sum_{i=1}^{t-1}
d_{\Pi}\!\bigl(\pi_{i+1}(\cdot\mid s),\pi_i(\cdot\mid s)\bigr).
\]
Taking the supremum over $s\in\mathcal S$ on both sides obtains
\begin{equation}
\sup_{s \in \mathcal{S}} d_\Pi(\hat{\pi}(\cdot\mid s),\pi_t(\cdot\mid s))
\le
d_\Pi(\hat{\pi},\pi_1)
+
\sum_{i=1}^{t-1} d_\Pi(\pi_{i+1},\pi_i).
\label{eq:policy_bound}
\end{equation}

Since $\pi_1(\cdot\mid s)$ is uniform over $\mathcal A$ for all $s\in\mathcal S$, we have for any $s$,
\begin{align*}
\mathrm{KL}\!\bigl(\hat{\pi}(\cdot\mid s)\,\|\,\pi_1(\cdot\mid s)\bigr)
&=
\sum_{a\in\mathcal A}\hat{\pi}(a\mid s)
\log\frac{\hat{\pi}(a\mid s)}{1/|\mathcal A|}\\
&=
\log|\mathcal A| +  \sum_{a\in\mathcal A}\hat{\pi}(a\mid s)
\log \hat{\pi}(a\mid s)\\
&\le
\log|\mathcal A|,
\end{align*}
By Pinsker's inequality, for any $s\in\mathcal S$,
\begin{equation*}
   d_{\Pi}\!\bigl(\hat{\pi}(\cdot\mid s),\pi_1(\cdot\mid s)\bigr)
\le
\sqrt{\tfrac12\,\mathrm{KL}\!\bigl(\hat{\pi}(\cdot\mid s)\,\|\,\pi_1(\cdot\mid s)\bigr)}
\le
\sqrt{\frac{\log|\mathcal A|}{2}}. 
% \label{eq:loga}
\end{equation*}
Taking the supremum over $s$, we have
\begin{equation}
d_\Pi(\hat{\pi},\pi_1)\le \sqrt{\frac{\log|\mathcal A|}{2}}.
\label{eq:loga}
\end{equation}

By assumption, for all $i=1,\dots,t-1$,
\[
\sup_{s\in\mathcal S}
\mathrm{KL}\!\bigl(\pi_{i+1}(\cdot\mid s)\,\|\,\pi_i(\cdot\mid s)\bigr)
\le
\alpha.
\]
Applying Pinsker's inequality again, we obtain for each $i$,
{
\begin{align*}
d_\Pi(\pi_{i+1},\pi_i)
&\le
\sup_{s\in\mathcal S}
\sqrt{\tfrac12\,
\mathrm{KL}\!\bigl(\pi_{i+1}(\cdot\mid s)\,\|\,\pi_i(\cdot\mid s)\bigr)}
\le
\sqrt{\frac{\alpha}{2}}.
\end{align*}
}
Consequently,
\begin{equation}
 \sum_{i=1}^{t-1} d_\Pi(\pi_{i+1},\pi_i)
\le
(t-1)\sqrt{\frac{\alpha}{2}}
=
\sqrt{\frac{(t-1)^2\alpha}{2}}.
\label{eq:each_policy}
\end{equation}

Combining the bounds in equations \eqref{eq:policy_bound}, \eqref{eq:loga}, and \eqref{eq:each_policy}, we conclude that
\[
d_\Pi(\hat{\pi},\pi_t)
\le
\sqrt{\frac{\log|\mathcal A|}{2}}
+
\sqrt{\frac{(t-1)^2\alpha}{2}},
\]
which completes the proof.
\end{proof}

Then, we restate and prove the policy drift bound in Lemma~\ref{mlem:policy_bound}.

\renewcommand{\thelemma}{4}
\begin{lemma}[policy drift bound]
\label{lem:policy_bound}
Under Assumptions~\ref{mass:lipschitz-policy} and~\ref{mass:klpolicy}, let $\mathcal A$ be a finite action space and $\hat{\pi} \in \Pi$ be a fixed policy. Set parameter $\alpha= 4\log{|\mathcal A|}/T^2$. At time step $h\in [H]$ over $T$ episodes, we have this policy drift bound,
\begin{align*}
    \sup_{\pi \in \Pi} \left|\mathbb{E}_{s_h\sim\mathcal D_h^{\hat{\pi}}} Q_h^{*}(s_h,a_h^\pi)
    -
    \frac{1}{T}\sum_{t=1}^T \mathbb{E}_{s_h^t \sim \mathcal D_{h}^{\pi_t}}Q_h^{*}(s_h^t,a_h^\pi)\right|\le
     L_\Pi \sqrt{2\log|\mathcal A|}.
  \end{align*}
\end{lemma}
\renewcommand{\thelemma}{\arabic{lemma}}

\begin{proof}
This term, \(
\sup_{\pi \in \Pi} \left|\mathbb{E}_{s_h\sim\mathcal D_h^{\hat{\pi}}} Q_h^{*}(s_h,a_h^\pi)
-
\frac{1}{T}\sum_{t=1}^T \mathbb{E}_{s_h^t \sim \mathcal D_{h}^{\pi_t}}Q_h^{*}(s_h^t,a_h^\pi)\right|
\), measures the discrepancy between the state distribution $\mathcal D_h^{\hat{\pi}}$ induced by the fixed policy $\hat{\pi}$ and the state distributions $\{\mathcal D_{h}^{\pi_t}\}_{t=1}^T$ induced by the learned policy $\pi_t$ over $T$ training episodes. 

We apply the IPM definition:
\[
\sup_{\pi \in \Pi} \left|\mathbb{E}_{s_h\sim\mathcal D_h^{\hat{\pi}}} Q_h^{*}(s_h,a_h^\pi)
-
\mathbb{E}_{s_h^t \sim \mathcal D_{h}^{\pi_t}}Q_h^{*}(s_h^t,a_h^\pi)\right| \leq D_{\mathcal{Q}_\Pi}(\mathcal{D}^{\hat{\pi}}_h,\mathcal{D}^{\pi_t}_h)
    ,\quad \forall t=1,\cdots,T.
\]
According to Assumption~\ref{mass:lipschitz-policy} and by the triangle inequality, we have
  \begin{align*}
    &\sup_{\pi \in \Pi} \left|\mathbb{E}_{s_h\sim\mathcal D_h^{\hat{\pi}}} Q_h^{*}(s_h,a_h^\pi)
    -
    \frac{1}{T}\sum_{t=1}^T \mathbb{E}_{s_h^t \sim \mathcal D_{h}^{\pi_t}}Q_h^{*}(s_h^t,a_h^\pi)\right|\\
    &=
    \sup_{\pi \in \Pi} \left|\frac{1}{T}\sum_{t=1}^T \left(\mathbb{E}_{s_h\sim\mathcal D_h^{\hat{\pi}}} Q_h^{*}(s_h,a_h^\pi)
    -
     \mathbb{E}_{s_h^t \sim \mathcal D_{h}^{\pi_t}}Q_h^{*}(s_h^t,a_h^\pi)\right)\right|
    \\
    &\le
    \frac{1}{T}\sum_{t=1}^T \left(
    \sup_{\pi \in \Pi}\left|\mathbb{E}_{s_h\sim\mathcal D_h^{\hat{\pi}}} Q_h^{*}(s_h,a_h^\pi)
    -
     \mathbb{E}_{s_h^t \sim \mathcal D_{h}^{\pi_t}}Q_h^{*}(s_h^t,a_h^\pi)\right|\right)
    \\
    &\le
    \frac{1}{T}\sum_{t=1}^T
    D_{\mathcal Q_\Pi}(\mathcal{D}^{\hat{\pi}}_h,\mathcal{D}^{\pi_t}_h)
    \\
    &\le
    \frac{L_\Pi}{T}
    \sum_{t=1}^T d_\Pi(\hat{\pi},\pi_t).
  \end{align*}
We apply Lemma~\ref{lem:policy_path_length} to bound the $d_\Pi(\hat{\pi},\pi_t)$, which decomposes the distance to the fixed policy into two components: 
the discrepancy between initial policy $\pi_1$ and fixed policy $\hat{\pi}$, and the cumulative step size of policy updates, each constrained by the KL divergence. 

For some $\alpha > 0$, $\sup_{s\in\mathcal S}\mathrm{KL}\!\bigl(\pi_{t+1}(\cdot\mid s)\,\|\,\pi_t(\cdot\mid s)\bigr)
\le \alpha,\quad \forall\, t=1,\dots,T.$ According to Lemma~\ref{lem:policy_path_length}, we have:
\begin{align*}
d_\Pi(\hat{\pi},\pi_t) \leq \sqrt{\frac{\log|\mathcal A|}{2}}+\sqrt{\frac{(t-1)^2 \alpha}{2}}.
\end{align*}
Therefore,
\begin{align*}
\frac{L_\Pi}{T} \sum_{t=1}^T d_\Pi(\hat{\pi},\pi_t) 
&\leq L_\Pi\sqrt{\frac{\log|\mathcal A|}{2}}+ \frac{L_\Pi}{T}\sum_{t=2}^T \sqrt{\frac{(t-1)^2 \alpha}{2}} \\
& \leq L_\Pi\sqrt{\frac{\log|\mathcal A|}{2}}+ L_\Pi\sqrt{\frac{T^2 \alpha}{8}}.
\end{align*}

Then, we set the parameter $\alpha= 4\log{|\mathcal A|}/T^2$ and obtain:
\begin{align*}
\frac{L_\Pi}{T} \sum_{t=1}^T \sup_{s\in\mathcal S}d_\Pi(\hat{\pi}(\cdot\mid s),\pi_t(\cdot\mid s))  \leq L_\Pi \sqrt{2\log|\mathcal A|},
\end{align*}
and completing the proof.
\end{proof}

We now define the Rademacher complexity of a function class $\mathcal{F}$ under non-independent and identically distributed (non-iid) setting \cite{Bartlett2003RademacherAG,liu2022understanding}.
\begin{definition}[Rademacher complexity under non-iid setting \cite{Bartlett2003RademacherAG,liu2022understanding}]
\label{def:radcomp}
Let $\mathcal{F} \subseteq \mathbb{R}^{\mathcal{S}}$ be a function class and
$\mathbf{s}^{1:n}=(s^1,\dots,s^n)$ be independent samples drawn from distributions $\mathcal{D}^1,\cdots,\mathcal{D}^n$.
Let $\bm{\sigma}^{1:n}=(\sigma^1,\dots,\sigma^n)$ be independent Rademacher random variables.
The \emph{Rademacher complexity} of $\mathcal{F}$ is defined as
\[
\mathfrak{R}(\mathcal{F})
\triangleq
\mathbb{E}_{\mathbf{s}^{1:n}}
\mathbb{E}_{\bm{\sigma}^{1:n}}
\left[
    \sup_{f\in\mathcal F}
    \frac{1}{n} \sum_{i=1}^n \sigma^i f(s^i)
\right].
\]
\end{definition}

 We then restate the on-average generalisation bound in Lemma~\ref{mlem:rad-noniid-full}
\renewcommand{\thelemma}{5}
\begin{lemma}[on-average generalisation bound]
\label{lem:rad-noniid-full}
Let $\mathbf s^{1:T}_h =\{s_h^1,\dots,s_h^T\}$ be independent
random variables with $s^t_h \sim \mathcal D_{h}^{\pi_t}$ on a space $\mathcal S$. Define the averaged state distribution
\(
    \bar{\mathcal D}_{h} \triangleq \frac{1}{T}\sum_{t=1}^T \mathcal D_{h}^{\pi_t},
\)
and the Rademacher complexity
\(
    \mathfrak{R}_h(\mathcal Q_\Pi)
\) of value function class $\mathcal Q_\Pi$. For any $\delta\in(0,1)$, with probability at least $1-\delta/2H$, we have: 
\begin{align*}
    \sup_{\pi \in \Pi}\left[\mathbb E_{s_h\sim\bar{\mathcal{D}}_{h}}[Q_h^{*}(s_h,a_h^\pi)]
        -
        \frac{1}{T}\sum_{t=1}^T Q_h^{*}(s_h^t,a_h^\pi)\right]
        \le
    2\mathfrak{R}_h(\mathcal Q_\Pi)
    +
    \sqrt{\frac{H^2\log(2H/\delta)}{2T}}.
\end{align*}
\end{lemma}
\renewcommand{\thelemma}{\arabic{lemma}}

\begin{proof}
The proof follows the classical symmetrisation techniques. 
We define
\[
    \Phi(s_h^1,\dots,s_h^T)
    \triangleq
    \sup_{f\in\mathcal Q_\Pi}
    \left\{
        \mathbb E_{s_h\sim\bar{\mathcal{D}}_{h}}[f(s_h)]
        -
        \frac{1}{T}\sum_{t=1}^T f(s^t_h)
    \right\}.
\]
Let $\tilde s_h^1,\dots,\tilde s_h^T$ be an independent ghost sample with $\tilde s^t_h\sim\mathcal D_{h}^{\pi_t}$.
Since
\[
    \mathbb E_{s_h \sim\bar{\mathcal D}_{h}}[f(s_h)]
    =
        \frac{1}{T}\sum_{t=1}^T \left[\mathbb E_{s_h^t \sim \mathcal{D}_h^{\pi_t}}f(s^t_h)
    \right]
    =
    \mathbb E_{\tilde s_h^1,\dots,\tilde s_h^T}\!\left[
        \frac{1}{T}\sum_{t=1}^T f(\tilde s^t_h)
    \right],
\]
we calculate the expectation of $\Phi$
\begin{align*}
\mathbb E[\Phi]
& =
\mathbb E_{s_h^1,\dots,s_h^T}\left[
\sup_{f\in \mathcal Q_\Pi}\left(
\mathbb E_{s_h \sim \bar{\mathcal D}_{h}}
    \left[f(s_h)\right]
    -
    \frac{1}{T}\sum_{t=1}^T f(s^t_h)
\right)\right]\\
&=
\mathbb E_{s_h^1,\dots,s_h^T}\left[
\sup_{f\in \mathcal Q_\Pi}\left(
\mathbb E_{\tilde s_h^1,\dots,\tilde s_h^T}
    \left[\frac{1}{T}\sum_{t=1}^T f(\tilde s^t_h)\right]
    -
    \frac{1}{T}\sum_{t=1}^T f(s^t_h)
\right)\right]\\
&=
\mathbb E_{s_h^1,\dots,s_h^T}\left[
\sup_{f\in \mathcal Q_\Pi}\left(
\mathbb E_{\tilde s_h^1,\dots,\tilde s_h^T}\left[
    \frac{1}{T}\sum_{t=1}^T f(\tilde s^t_h)
    -
    \frac{1}{T}\sum_{t=1}^T f(s^t_h)\right]
\right)\right],
\end{align*}

Jensen's inequality gives
\[
\mathbb E[\Phi]
\le
\mathbb E_{s_h^1,\dots,s_h^T}\left[\mathbb E_{\tilde s_h^1,\dots,\tilde s_h^T}
\left[
    \sup_{f\in \mathcal Q_\Pi}
    \frac{1}{T}\sum_{t=1}^T (f(\tilde s^t_h)-f( s^t_h))
\right]\right].
\]
Introduce independent Rademacher variables $\bm \sigma_h^{1:T}=\{\sigma_h^1,\dots,\sigma_h^T\}$ where $\sigma_h^{1:T} \in \{-1,1\}^T$ with probability of $1/2^T$.
By symmetry of the Rademacher variables,
\begin{align}
\label{eq:rade_phi}
&\mathbb E_{s_h^1,\dots,s_h^T}\left[\mathbb E_{\tilde s_h^1,\dots,\tilde s_h^T}
\left[
    \sup_{f\in \mathcal Q_\Pi}
    \frac{1}{T}\sum_{t=1}^T (f(\tilde s^t_h)-f(s^t_h))
\right]\right] \nonumber \\ 
& =
\mathbb E_{s_h^1,\dots,s_h^T}\left[\mathbb E_{\tilde s_h^1,\dots,\tilde s_h^T}
\left[ \mathbb E_{\bm \sigma_h^{1:T}} \left[
    \sup_{f\in \mathcal Q_\Pi}\left(
    \frac{1}{T}\sum_{t=1}^T \sigma_h^t(f(\tilde s^t_h)-f(s^t_h))\right)\right]
\right]\right] \nonumber\\
& \le
2\,\mathbb E_{\mathbf s_{h}^{1:T}}\mathbb E_{\bm{\sigma}_h^{1:T}}
\left[
    \sup_{f\in \mathcal Q_\Pi}
    \frac{1}{T}\sum_{t=1}^T \sigma_h^t f(s^t_h)
\right] \nonumber\\
&= 2\mathfrak R_h(\mathcal Q_\Pi).
\end{align}

Then, we apply McDiarmid's inequality. In EMDP with bounded reward $0\leq r_h \leq 1$, the value function satisfies that $|f(s)-f(s')|\leq H$ for any $s,s' \in \mathcal{S}$. If two samples $\mathbf s_h^{1:T}$ and $\mathbf s_h^{\prime 1:T}$ differ only in the $t$-th episode, then
\[
|\Phi(\mathbf s_h^{1:T})-\Phi(\mathbf s_h^{\prime 1:T})|
\le
\sup_{f\in\mathcal Q_\Pi}
\frac{|f(s_h^t)-f(s_h^{\prime t})|}{T}
\le
\frac{H}{T}.
\]
Hence $\Phi(s_h^1,\dots,s_h^T)$ satisfies bounded differences with
$c_t = H/T$. 

Therefore, we have
\[
\Pr\bigl(\Phi - \mathbb E[\Phi] \ge \epsilon\bigr)
\le
\exp\!\left(
    -\frac{2\epsilon^2}{\sum_{t=1}^T c_t^2}
\right)
=
\exp\!\left(
    -\frac{2T\epsilon^2}{H^2}
\right).
\]
Setting the Right-Hand Side (RHS) equal to $\delta/2H$ obtains
\[
\epsilon
=
\sqrt{\frac{H^2\log(2H/\delta)}{2T}}.
\]
Combining this bound with $\mathbb E[\Phi]\le 2\mathfrak R_h(\mathcal Q_\Pi)$
completes the proof.
\end{proof}

We combine the upper bound on policy drift in Lemma~\ref{lem:policy_bound} and the above bound in Lemma~\ref{lem:rad-noniid-full} to prove the restated Theorem~\ref{mthm:multi-source-rl}.

\renewcommand{\thetheorem}{2}
\begin{theorem}[intrinsic rational gap bound]
  \label{thm:multi-source-rl}
  Under Assumptions~\ref{mass:lipschitz-env},~\ref{mass:lipschitz-policy} and~\ref{mass:klpolicy}, let $\hat{\mathfrak{R}}_h(\mathcal{Q}_\Pi)$ denote the empirical Rademacher complexity of value function class $\mathcal{Q}_\Pi$ with a sequence of states $\mathbf s_h^{1:T}=\{s_h^t\}_{t=1}^T$ at time step $h\in [H]$. For any $\delta\in(0,1)$, with probability at least $1-\delta$, the upper bound on intrinsic rational gap is:
  \begin{align*}
      \sum_{h=1}^H\sup_{\pi \in \Pi} \left| \mathbb{E}_{s_h \sim \mathcal D_h^{\hat{\pi}}} Q_h^*(s_h,a_h^{\pi})
      -
      \frac{1}{T}\sum_{t=1}^T Q_h^*(s_h^t,a_h^{\pi}) \right| 
 \leq 
    L_\Pi H \sqrt{2\log|\mathcal A|}+2\sum_{h=1}^H\hat{\mathfrak{R}}_h(\mathcal{Q}_\Pi)+3H^2\sqrt{\frac{\log(4H/\delta)}{2T}}.
  \end{align*}
\end{theorem}
\renewcommand{\thetheorem}{\arabic{theorem}}

\begin{proof}
Taking the supremum over $f \in \mathcal Q_\Pi$ and using the triangle inequality,
we obtain
\begin{align}
\sup_{f \in \mathcal Q_\Pi} \left| \mathbb{E}_{s_h \sim \mathcal D_h^{\hat{\pi}}} f(s_h)
      -
      \frac{1}{T}\sum_{t=1}^T f(s^t_h) \right|
\le&
\underbrace{\sup_{f \in \mathcal Q_\Pi} \left| \mathbb{E}_{s_h \sim \mathcal D_h^{\hat{\pi}}} f(s_h)
      -
      \frac{1}{T}\sum_{t=1}^T\mathbb{E}_{s^t_h \sim \mathcal D_{h}^{\pi_t}} f(s^t_h) \right|}_{\text{Term I}} \nonumber\\
&+
   \underbrace{\sup_{f \in \mathcal Q_\Pi} \left| \frac{1}{T}\sum_{t=1}^T \left(\mathbb{E}_{s^t_h \sim \mathcal D_{h}^{\pi_t}} f(s^t_h)
    - f(s^t_h)\right) \right|}_{\text{Term II}}.
\label{eq:AB-decomp}
\end{align}

{According to Lemma~\ref{lem:policy_bound}}, we have
\begin{equation}
\label{eq:A-bound}
\text{Term}\; \mathrm I
\;\le\;
L_\Pi\sqrt{2\log|\mathcal A|}.
\end{equation}

Since $\mathfrak R_h(-\mathcal Q_\Pi)=\mathfrak R_h(\mathcal Q_\Pi)$, a union
bound over the two directions of one-sided bound in Lemma~\ref{lem:rad-noniid-full} gives that for any $\delta\in(0,1)$, with probability at least $1-\delta/2H$,
\begin{equation}
\label{eq:B-pop-R}
\text{Term}\; \mathrm{II} \;\le\;
2\,\mathfrak{R}_h(\mathcal Q_\Pi)
+
\sqrt{\frac{H^2 \log(4H/\delta)}{2T}},
\end{equation}
where the Rademacher complexity in Definition~\ref{def:radcomp}  is
\[
\mathfrak{R}_h(\mathcal Q_\Pi)
\triangleq
\mathbb{E}_{\mathbf{s}_{h}^{1:T}}\left[
\mathbb{E}_{\bm \sigma_h^{1:T}}
\left[
\sup_{f \in \mathcal Q_\Pi}
\frac{1}{T}\sum_{t=1}^T \sigma_h^t f(s^t_h)
\right]
\right]
.
\]

We now replace it with the empirical Rademacher
complexity with a sequence of states $\mathbf s_h^{1:T} = \{s_h^t\}^T_{t=1}$ in equation (\ref{eq:B-pop-R}).
\[
\hat{\mathfrak{R}}(\mathcal Q_\Pi, \mathbf s_h^{1:T})
\triangleq
\mathbb{E}_{\bm \sigma_h^{1:T}}
\left[
\sup_{f \in \mathcal Q_\Pi}
\frac{1}{T}\sum_{t=1}^T \sigma_h^t f(s^t_h)
\right].
\]
By definition,
\(
\mathfrak{R}_h(\mathcal Q_\Pi)
=
\mathbb E_{\mathbf s_{h}^{1:T}}\big[\hat{\mathfrak{R}}_h(\mathcal Q_\Pi, \mathbf s_h^{1:T})\big].
\)

If two samples $\mathbf s_h^{1:T}$ and $\mathbf s_h^{\prime 1:T}$ differ only in the $t$-th episode, then
\[
|\hat{\mathfrak{R}}(\mathcal Q_\Pi,\mathbf s_h^{1:T})-\hat{\mathfrak{R}}(\mathcal Q_\Pi,\mathbf s_h^{\prime 1:T})|
\le
\sup_{f\in\mathcal Q_\Pi}
\frac{|f(s_h^t)-f(s_h^{\prime t})|}{T}
\le
\frac{H}{T}.
\]

Hence, by McDiarmid's inequality, for any $\epsilon>0$,
\[
\Pr\Big(
\mathbb E_{\mathbf{s}_h^{1:T}}[\hat{\mathfrak{R}}(\mathcal Q_\Pi, \mathbf s_h^{1:T})]
-
\hat{\mathfrak{R}}(\mathcal Q_\Pi,\mathbf s_h^{1:T})
\ge \epsilon
\Big)
\le
\exp\!\left(-\frac{2T\epsilon^2}{H^2}\right).
\]
Setting
\(
\epsilon
=
H\sqrt{\log(2H/\delta)/2T}
\)
and recalling that
\(
\mathfrak{R}_h(\mathcal Q_\Pi)
=
\mathbb E_{\mathbf s^{1:T}_{h}}\big[\hat{\mathfrak{R}}(\mathcal Q_\Pi, \mathbf s_h^{1:T})\big]
\),
we obtain that with probability at least $1-\delta/2H$,
\begin{equation}
\label{eq:R-vs-Rhat}
\mathfrak{R}_h(\mathcal Q_\Pi) -\hat{\mathfrak{R}}(\mathcal Q_\Pi, \mathbf s_h^{1:T})
\leq
\sqrt{\frac{H^2\log(2H/\delta)}{2T}}.
\end{equation}

Combining equation \eqref{eq:B-pop-R} and equation \eqref{eq:R-vs-Rhat}, and using a union bound over the two probabilistic events and $\log(2H/\delta)\le \log(4H/\delta)$, we obtain that with
probability at least $1-\delta/H$,
\begin{equation}
\label{eq:B-final}
\text{Term}\; \mathrm{II} \;\le\;
2\,\hat{\mathfrak{R}}(\mathcal Q_\Pi,\mathbf s_h^{1:T})
+
3\sqrt{\frac{H^2\log(4H/\delta)}{2T}}.
\end{equation}

Substituting equation \eqref{eq:A-bound} and equation \eqref{eq:B-final} into the decomposition in equation
\eqref{eq:AB-decomp} and using a union bound over all $h\in[H]$, let $\hat{\mathfrak{R}}_h(\mathcal Q_\Pi)=\hat{\mathfrak{R}}(\mathcal Q_\Pi,\mathbf s_h^{1:T})$, we conclude that with probability at least $1-\delta$, 
\begin{align*}
\sum_{h=1}^H\sup_{f \in \mathcal Q_\Pi}
\left| \mathbb{E}_{s_h^t \sim \mathcal D_h^{\hat{\pi}}} f(s_h^t)
      -
      \frac{1}{T}\sum_{t=1}^T f(s^t_h) \right|
&\le
\sum_{h=1}^H\left[
L_\Pi\sqrt{2\log|\mathcal A|}
+2\,\hat{\mathfrak{R}}_{h}(\mathcal Q_\Pi)
+3H\sqrt{\frac{\log(4H/\delta)}{2T}}
\right]\\
&\le L_\Pi H\sqrt{2\log|\mathcal A|}+2\sum_{h=1}^H\hat{\mathfrak{R}}_h(\mathcal{Q}_\Pi)+3H^2\sqrt{\frac{\log(4H/\delta)}{2T}}.
\end{align*}
We complete the proof.

\end{proof}

\section{Proof of Corollary \ref{mlem:exdqn} and Corollary \ref{mcor:extension}}
\label{sec:cor}
This section restates Corollary \ref{lem:exdqn} and Corollary \ref{cor:extension} and presents their proofs.
\setcounter{theorem}{3}
\begin{corollary}[rational value metric bound]
\label{lem:exdqn}
Assuming that the learned value function $\hat{Q}_h^T$ approximates optimal value function $Q_h^*$ with a bounded $\ell_\infty$ error of $\epsilon$, $\|Q_h^*-\hat{Q}_h^T\|_{\infty}\leq \epsilon$ for all $h \in [H]$. For any $\delta \in (0,1)$ with probability at least $1-\delta$, its rational risk gap of policy $\hat{\pi} \in \Pi$ over $T$ episodes of horizon $H$ can be bounded by:
\begin{align*}
\bigl|\mathcal R(\hat{\pi})-\hat{\mathcal{R}}_{\mathrm{ALG}}(\hat{\pi})\bigr|\leq&
\beta_1\cdot W_1(p_0^\dagger,p_0)+\beta_2\cdot W_1(p^\dagger, p)
+4\sum_{h=1}^H\hat{\mathfrak{R}}_h(\mathcal{Q}_\Pi) \\&+ 6H^2\sqrt{\frac{\log(4H/\delta)}{2T}}+
2 H\epsilon+2L_\Pi H \sqrt{2\log|\mathcal A|},
\end{align*}
where $\beta_1=2L_sH$ and $\beta_2=2H^2L_s(L_p+1)$.
\end{corollary}}
\begin{proof}
By adding and subtracting the empirical risk computed with $Q_h^*$, we have
\begin{align}
\left|\mathcal R(\hat\pi)-\hat{\mathcal R}_{\mathrm{ALG}}(\hat\pi)\right|
\le
\underbrace{\left|\mathcal R(\hat\pi)-\hat{\mathcal R}(\hat\pi)\right|}_{\text{Term I}}
+
\underbrace{\left|\hat{\mathcal R}(\hat\pi)-\hat{\mathcal R}_{\mathrm{ALG}}(\hat\pi)\right|}_{\text{Term II}}.
\label{eq:rvm}
\end{align}
By Theorem~\ref{mthm:main_theory}, for any $\delta \in(0,1)$, with probability at least $1-\delta$,
\[
\left|\mathcal R(\hat\pi)-\hat{\mathcal R}(\hat\pi)\right|
\le
\beta_1 W_1(p_0^\dagger,p_0)
+
\beta_2 W_1(p^\dagger,p)
+
4\sum_{h=1}^H\hat{\mathfrak R}_h(\mathcal Q_\Pi)
+
6H^2\sqrt{\frac{\log(4H/\delta)}{2T}}
+
2L_\Pi H\sqrt{2\log|\mathcal A|}.
\]
It remains to control the approximation error. Since
$\|Q_h^*-\hat Q_h^T\|_\infty\le\epsilon$ for all $h\in[H]$, we add $\max_{a\in\mathcal{A}} \hat{Q}^T_h(s_h^t,a)-\hat{Q}^T_h(s_h^t,a^\circ)\geq0$ in Term II of equation \eqref{eq:rvm},
\begin{align*}
&\left|\hat{\mathcal R}(\hat\pi)-\hat{\mathcal R}_{\mathrm{ALG}}(\hat\pi)\right|\\
&=\left|
\frac1T\sum_{t=1}^T\sum_{h=1}^H\left[
Q_h^*(s_h^t,a_h^\circ)
-
Q_h^*(s_h^t,a_h^{\hat{\pi}})
-
\max_{a\in\mathcal{A}}\hat Q_h^T(s_h^t,a)
+
\hat Q_h^T(s_h^t,a_h^{\hat{\pi}})
\right]\right|
\\
&\le\left|
\frac1T\sum_{t=1}^T\sum_{h=1}^H\left[
Q_h^*(s_h^t,a_h^\circ)
-
Q_h^*(s_h^t,a_h^{\hat{\pi}})
-
\max_{a\in\mathcal{A}}\hat Q_h^T(s_h^t,a)
+
\hat Q_h^T(s_h^t,a_h^{\hat{\pi}})+\max_{a\in\mathcal{A}}\hat Q_h^T(s_h^t,a)
-
\hat Q_h^T(s_h^t,a_h^{\circ})
\right]\right|\\
&\le
2\sum_{h=1}^H
\sup_{\pi\in\Pi}
\left|
\frac1T\sum_{t=1}^T
Q_h^*(s_h^t,a_h^\pi)
-
\frac1T\sum_{t=1}^T
\hat Q_h^T(s_h^t,a_h^\pi)
\right|\\
&\le
2\sum_{h=1}^H
\sup_{\pi\in\Pi}
\frac1T\sum_{t=1}^T
\left|
Q_h^*(s_h^t,a_h^\pi)
-
\hat Q_h^T(s_h^t,a_h^\pi)
\right|\\
&\le
2H\epsilon.
\end{align*}
Combining the two bounds gives the result.
\end{proof}
{
\setcounter{theorem}{4}
\begin{corollary}[rational risk gap bound under reward shift]
\label{cor:extension}
Under the Assumption~\ref{ass:rsb} and the condition of Theorem~\ref{mthm:main_theory}, Suppose $\mathcal{R}'(\pi)$ denotes the expected rational value risk under $r'_h(s,a)$, the rational risk gap $\left|\mathcal{R}'(\pi)-\hat{\mathcal{R}}(\pi)\right|$ of policy $\pi \in \Pi$ over a trajectory of horizon $H$ can be decomposed as follows,
\begin{align*}
\bigl|\mathcal R'(\hat{\pi})-\hat{\mathcal{R}}(\hat{\pi})\bigr|\leq&
\beta_1\cdot W_1(p_0^\dagger,p_0)+\beta_2\cdot W_1(p^\dagger, p)+2L_\Pi H\cdot \sqrt{2\log|\mathcal A|}
\\&+4\sum_{h=1}^H\hat{\mathfrak{R}}_h(\mathcal{Q}_\Pi)+6H^2\sqrt{\frac{\log(4H/\delta)}{2T}}+
H(H+1)\varphi,
\end{align*}
where $\beta_1=2L_sH$ and $\beta_2=2H^2L_s(L_p+1)$.
\end{corollary}}
}
{
\begin{proof}
We first claim that for all $h \in [H]$ and all $(s,a) \in \mathcal{S} \times \mathcal{A}$,
\begin{equation}
\left| Q_h^{\prime,*,\dagger}(s,a) - Q_h^{*,\dagger}(s,a) \right| \leq (H - h + 1)\varphi.
\label{eq:indq}
\end{equation}
We prove it by backward induction. We consider the base case of $h =H$. By the Bellman equation and the terminal condition that $V_{H+1}^{*,\dagger}(s) = 0$ for all $s$,
$$Q_H^{*,\dagger}(s,a) = r_H(s,a), \qquad Q_H^{\prime,*,\dagger}(s,a) = r^\prime_H(s,a).$$
Hence, by Assumption~\ref{ass:rsb},
$$\left| Q_H^{\prime,*,\dagger}(s,a) - Q_H^{*,\dagger}(s,a) \right| = \left|r^\prime_H(s,a) - r_H(s,a) \right| \leq \varphi = (H - H + 1)\varphi,$$
which establishes the base case.
Now suppose that equation \eqref{eq:indq} holds for all $(s,a) \in \mathcal{S}\times\mathcal{A}$ at step $h+1$.
By the definition $V_{h+1}^{*,\dagger}(s) = \mathbb{E}_{a \sim \pi^*(\cdot \mid s)}\left[Q_{h+1}^{*,\dagger}(s,a)\right]$, we have
\begin{align*}
\left|V_{h+1}^{\prime,*,\dagger}(s) - V_{h+1}^{*,\dagger}(s) \right|
&= \left| \mathbb{E}_{a \sim \pi^*(\cdot \mid s)}\left[ Q_{h+1}^{\prime,*,\dagger}(s,a) - Q_{h+1}^{*,\dagger}(s,a) \right] \right| \\
&\leq \mathbb{E}_{a \sim \pi^*(\cdot \mid s)}\left| Q_{h+1}^{\prime,*,\dagger}(s,a) - Q_{h+1}^{*,\dagger}(s,a) \right| \\
&\leq (H - h)\varphi,
\end{align*}
where the last inequality applies the inductive hypothesis.
By the Bellman equation,
\begin{align*}
Q_{h}^{*,\dagger}(s,a) = r_{h}(s,a) + \mathbb{E}_{s' \sim p^\dagger(\cdot \mid s,a)}\left[ V_{h+1}^{*,\dagger}(s') \right], \qquad
Q_{h}^{\prime,*,\dagger}(s,a) = r^\prime_{h}(s,a) + \mathbb{E}_{s' \sim p^\dagger(\cdot \mid s,a)}\left[V_{h+1}^{\prime,*,\dagger}(s') \right].
\end{align*}
Taking the difference and applying the triangle inequality,
\begin{align*}
\left|Q_{h}^{\prime,*,\dagger}(s,a) - Q_{h}^{*,\dagger}(s,a) \right|
&\leq \left| r^\prime_{h}(s,a) - r_{h}(s,a) \right|  + \mathbb{E}_{s' \sim p^\dagger(\cdot \mid s,a)}\left| V_{h+1}^{\prime,*,\dagger}(s') - V_{h+1}^{*,\dagger}(s') \right| \\
&\leq \varphi + (H - h)\varphi \\
&= (H - h + 1)\varphi,
\end{align*}
This completes the induction.
Applying the inductive hypothesis, for each $h \in [H]$,
\begin{align*}
\sup_{\pi \in \Pi}\left| \mathbb{E}_{s_h \sim \mathcal{D}_h^{\hat{\pi},\dagger}}Q_h^{\prime,*,\dagger}(s_h,a_h^\pi) - \mathbb{E}_{s_h \sim \mathcal{D}_h^{\hat{\pi},\dagger}}Q_h^{*,\dagger}(s_h,a_h^\pi) \right| 
&\leq \sup_{\pi \in \Pi}\,\mathbb{E}_{s_h \sim \mathcal{D}_h^{\hat{\pi},\dagger}}\left|Q_h^{\prime,*,\dagger}(s_h,a_h^\pi) - Q_h^{*,\dagger}(s_h,a_h^\pi) \right| \\
&\leq (H - h + 1)\varphi.
\end{align*}
Finally, we have,
$$2\sum_{h=1}^{H} \sup_{\pi \in \Pi}\left| \mathbb{E}_{s_h \sim \mathcal{D}_h^{\hat{\pi},\dagger}}Q_h^{\prime,*,\dagger}(s_h,a_h^\pi) - \mathbb{E}_{s_h \sim \mathcal{D}_h^{\hat{\pi},\dagger}}Q_h^{*,\dagger}(s_h,a_h^\pi) \right| \leq 2\varphi\sum_{h=1}^{H}(H - h + 1) = H(H+1)\varphi.$$
\end{proof}
}

\section{Additional Details of Environments}
\label{app:A}

This appendix presents the tables describing the environments and their division into training and inference settings. Table \ref{tab:env} summarises the key differences across environments, including state and action space dimensions. Tables \ref{tab:taxi} and \ref{tab:cliff} show the components for Taxi-v3 and Cliff Walking environments.

\begin{table}[h!]
\centering
\caption{Environment information}
\label{tab:env}
\begin{tabular}{l c c c}
\toprule
& State Space & Action Space & Reward Space \\
\midrule \midrule
Taxi-v3 
& $\mathcal{S} \in \{0,\dots,499\}$ 
& $\mathcal{A} \in \{0,\dots,5\}$ 
& $R \subset \mathbb{R}$ \\
CliffWalking-v0 
& $\mathcal{S} \in \{0,\dots,47\}$ 
& $\mathcal{A} \in \{0,\dots,3\}$ 
& $R \subset \mathbb{R}$ \\
\bottomrule
\end{tabular}
\end{table}

\begin{table}[h!]
\centering
\caption{State description for Taxi-v3.}
\label{tab:taxi}
\begin{tabular}{l c c}
\toprule
Index & State Component & Description \\
\midrule \midrule
0 & Taxi row position & Discrete grid row index \\
1 & Taxi column position & Discrete grid column index \\
2 & Passenger location & One of four landmarks or in taxi \\
3 & Destination & One of four landmarks \\
\bottomrule
\end{tabular}
\end{table}

\begin{table}[h!]
\centering
\caption{State description for Cliff Walking.}
\label{tab:cliff}
\begin{tabular}{l c c}
\toprule
Index & State Component & Description \\
\midrule \midrule
0 & Agent row position & Discrete grid row index \\
1 & Agent column position & Discrete grid column index \\
\bottomrule
\end{tabular}
\end{table}

\textbf{Environment shifts in training and inference} Since the original Taxi and Cliff Walking environments do not distinguish between training and inference settings, we introduce an \emph{action randomisation} mechanism with probability $0 \le \varepsilon \le 1$ to construct two distinct transition kernels, $p$ and $p^{\dagger}$. 
During training ($\varepsilon > 0$), the environment executes the agent’s taken action with probability $1 - \varepsilon$ and replaces it with a uniformly random action with probability $\varepsilon$, which leads to a perturbed transition kernel $p$. 
During inference ($\varepsilon = 0$), no action randomisation is applied, and the agent is evaluated under the original transition kernel $p^{\dagger}$. 
The pseudo-code for constructing the training transition kernel is provided in Algorithm~\ref{algo:slip_kernel}.

\begin{algorithm}[t]
\caption{\footnotesize Action randomisation for training and inference environment}
\label{algo:slip_kernel}
\begin{algorithmic}[1]
\REQUIRE Base transition kernel $p^\dagger(\cdot \mid s,a)$, number of states $n_S$, number of actions $n_A$, slip probability $\varepsilon \in [0,1]$
\ENSURE Training transition kernel $p(\cdot \mid s,a)$
\vspace{0.5em}
\STATE \textit{// Each kernel is represented as a set of tuples $(\rho, s', r, \text{done})$, where $\rho$ denotes transition probability; similarly $(q, s', r, \text{done}) \in \bar p(\cdot \mid s)$ with $q = \bar p(s',r,\text{done}\mid s)$.}
\STATE \textit{// Compute averaged transition kernel}
\FOR{$s = 0,1,\ldots,n_S-1$}
  \STATE Initialise empty kernel $\bar p(\cdot \mid s)$
  \FOR{$a = 0,1,\ldots,n_A-1$}
    \FORALL{$(\rho, s', r, \text{done}) \in p^\dagger(\cdot \mid s,a)$}
      \STATE $\bar p(s', r, \text{done} \mid s) \gets \bar p(s', r, \text{done} \mid s) + \frac{1}{n_A} \rho$
    \ENDFOR
  \ENDFOR
\ENDFOR

\STATE \textit{// Randomise actions}
\FOR{$s = 0,1,\ldots,n_S-1$}
  \FOR{$a = 0,1,\ldots,n_A-1$}
    \STATE Initialise empty kernel $p(\cdot \mid s,a)$
    \FORALL{$(\rho, s', r, \text{done}) \in p^\dagger(\cdot \mid s,a)$}
      \STATE $p(s', r, \text{done} \mid s,a) \gets (1-\varepsilon)\, \rho$
    \ENDFOR
    \FORALL{$(q, s', r, \text{done}) \in \bar p(\cdot \mid s)$}
      \STATE $p(s', r, \text{done} \mid s,a) \gets p(s', r, \text{done} \mid s,a) + \varepsilon\, q$
    \ENDFOR
  \ENDFOR
\ENDFOR

\STATE \textbf{Output:} $p$
\end{algorithmic}
\end{algorithm}

\subsection{Additional Details of Experimental Settings}
\label{app:B}
\textbf{Hyperparameters}
We report the hyperparameters of DQN used in our experiments. All other hyperparameters are left at their default settings.

\begin{table}[h!]
\centering
\caption{Hyperparameters for Deep-Q Network.}
\label{tab:hyperparams2}
\begin{tabular}{l c c}
\toprule
Hyperparameters &  \\
\midrule \midrule
Batch size &  64\\
Replay buffer size &  50,000\\
Softmax temperature $\tau$ & $10^{-7}$\\
Episodes &  5,000 \\
Warm-up steps & 1,000\\
Learning rate & 0.001 \\
Target network update period & 500 \\
Optimiser & Adam \\
Hidden dimension & 128\\
Initial exploration rate & 1.0\\
Final exploration rate & 0.05\\
Exploration decay episode & 3,000 \\
\bottomrule
\end{tabular}
\end{table}

\textbf{Computing the actual value functions $Q^{*,\dagger}_h(\cdot,\cdot)$ and $Q^{*}_h(\cdot,\cdot)$}
The tabular setting allows us to compute the actual value functions under both the inference and training environments, which is generally intractable in non-tabular reinforcement learning settings. We compute $Q^{*,\dagger}_h$ and $Q^{*}_h$ exactly by backward induction and derive the corresponding optimal policy $\pi^*$. To match our experimental setting, we further implement the optimal policy using a softmax parameterisation with a fixed temperature.

\paragraph{Justification on Calculating Expected Rational Value Risk and Rational Risk Gap} The selected environments provide access to the state distributions in both training and deployment, which enables the calculation of the expected rational value risk, and further the rational risk gap.

\subsection{Additional Empirical Results}
\paragraph{Relationship between Reward and Rational Risk Gap under Different $\ell_2$ Regularisation Strengths} We quantify the correlation between the rational risk gap and the episode reward using the Pearson correlation coefficient. Table \ref{tab:relation_l2} shows that they have a strong negative correlation. Compared to the original DQN, $\ell_2$ regularisation reduces the rational risk gap while improving reward, which suggests that an appropriate regularisation strength benefits both reward maximisation and rationality.

\begin{table}[h]
\centering
\caption{Pearson coefficient between rational risk gap and reward under different $\ell_2$ regularisation strengths. We evaluate DQN under increasing regularisation strengths ($10^{-3}$,$10^{-4}$,$10^{-5}$,$10^{-6}$,$10^{-7}$)}
\begin{tabular}{lccc|ccc}
\toprule
Variable
& \multicolumn{3}{c}{Taxi} 
& \multicolumn{3}{c}{Cliff Walking} \\
\cmidrule(lr){2-4} \cmidrule(lr){5-7}
& Rational risk gap & Reward & Coefficient 
& Rational risk gap & Reward & Coefficient \\
\midrule
DQN       
& $35.34\pm22.91$ & $-42.84\pm7.83$  & $-0.40$ 
& $206.67\pm26.50$ & $-96.34\pm27.72$ & $-0.32$\\

$10^{-3}$  
& ${15.06\pm5.71}$ & $-16.16\pm7.53$ & $-0.41$ 
& ${150.13\pm18.08}$ & $-43.24\pm15.60$ & $-0.53$\\

$10^{-4}$  
& $17.24\pm5.66$ & $-18.08\pm6.85$ & $-0.55$ 
& $204.48\pm24.52$ & $-41.86\pm9.41$ & $-0.55$ \\

$10^{-5}$  
& $26.95\pm25.04$ & ${-13.54\pm4.04}$ & $-0.48$ 
& $167.81\pm12.83$ & $-41.74\pm10.03$ & $-0.43$ \\

$10^{-6}$  
& $19.31\pm6.79$  & $-15.02\pm4.80$  & $-0.54$ 
& $206.66\pm46.32$ & $-41.84\pm24.97$  &  $-0.43$ \\

$10^{-7}$  
& $16.16\pm5.55$ & $-22.18\pm4.54$ & $-0.45$
& $162.72\pm16.84$ & $-41.18\pm13.08$ & $-0.56$\\

\bottomrule
\label{tab:relation_l2}
\end{tabular}
\end{table}

\paragraph{Experiment of a Special Case} In this section, we introduce a special case of expected rational value risk given the state distribution $\mathcal{D}_h^{*,\dagger}$ in deployment induced by optimal policy $\pi^*$ over a trajectory of horizon $H$, which is defined as
\[
\mathcal{R}^*(\pi)
\triangleq
\sum_{h=1}^H\mathbb E_{s_h \sim\mathcal D_h^{*,\dagger}}
\left[
Q_h^{*,\dagger}(s_h,a^{\circ}_h)-Q_h^{*,\dagger}(s_h,a^\pi_h)\right].
\]

% special case of rational risk gap of DQN 

% the reward curves and empirical rational value risk of DQN across five challenge levels of the Taxi-v3 and Cliff Walking environments in training, as shown in Figure~\ref{apfig:six_figures_level} and Figure~\ref{apfig2:six_figures_level}. They indicate that (1) the training in all cases is running well, and (2) rationality has a negative correlation with the environment shifts.

We adopt our experiment settings in Section~\ref{sec:imp} to measure a special case of rational risk gap as $|\mathcal{R}^*(\pi)-\hat{\mathcal{R}}(\pi)|$. When evaluating DQN under different regularisation and domain randomisation techniques, we consider a harder challenge level of both environments (from $10\%$ to $25\%$). Figure \ref{fig:reward} shows the reward curves in this setting, suggesting the training is running well. Figure \ref{fig:reg} and Figure
\ref{fig:dr} illustrate that DQN with $\ell_2$-regularisation and domain randomisation consistently reduce the rational risk gap across both environments relative to the original DQN, while layer normalisation and weight normalisation exhibit similar trends. Figure \ref{fig:six_figures_level} indicates that $|\mathcal{R}^*(\pi)-\hat{\mathcal{R}}(\pi)|$ has a negative correlation with the environment shifts. These results are consistent with the analysis in Section~\ref{sec:imp}.

\begin{figure}[t]
    \centering
    \includegraphics[width=\linewidth]{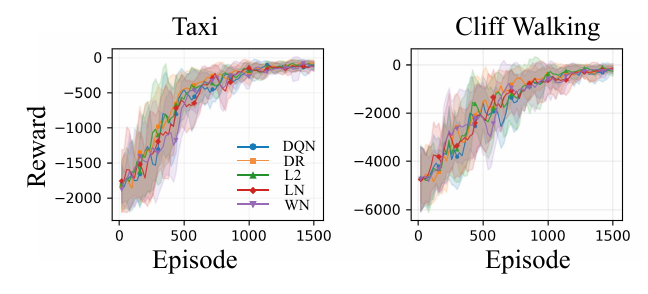}
    \caption{Reward curves of DQN under different regularisation and domain randomisation techniques in Taxi-v3 and Cliff Walking environments with challenge level
of $25\%$.}
    \label{fig:reward}
\end{figure}

\begin{figure*}[t!]
  \centering
  \begin{subfigure}[t]{0.9\textwidth}
    \centering
    \includegraphics[width=\linewidth]{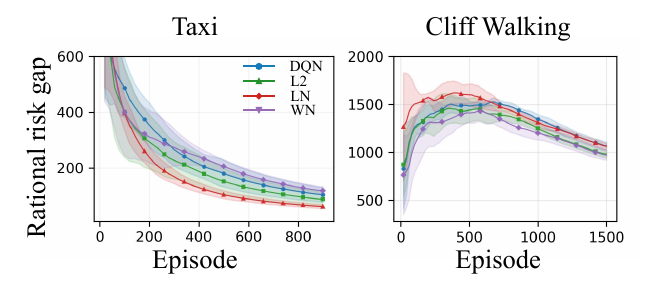}
    \caption{Regularisation}
    \label{fig:reg}
  \end{subfigure}
  \hfill
  \begin{subfigure}[t]{0.9\textwidth}
    \centering
    \includegraphics[width=\linewidth]{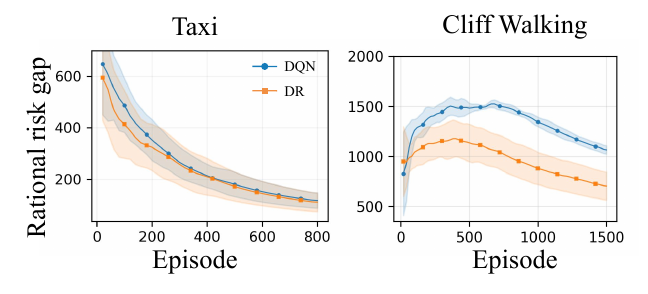}
    \caption{Domain randomisation}
    \label{fig:dr}
  \end{subfigure}
  \caption{Special case of rational risk gap of DQN under different regularisation and domain randomisation techniques in Taxi-v3 and Cliff Walking environments with challenge level
of $25\%$. {At the beginning of training, the policy is still incapable and thus causes frequent cliff falls (meaning a -100 penalty), so the terminal condition is triggered very quickly. Consequently, episodes have short horizons, making both the empirical rational risk and the rational risk gap small. As the agent learns to avoid falling off the cliff, but still fails to arrive at the target, the rational risk gap starts to increase. When the agents are able to reach the goal, the rational risk gap begins to decrease steadily. Instead, in Taxi-v3, the agent does not frequently trigger the terminal condition at the early stage, because of the environment properties.}}
  \label{fig:six_figures_reg}
\end{figure*}

\begin{figure}[t]
  \centering
  \includegraphics[width=0.95\linewidth]{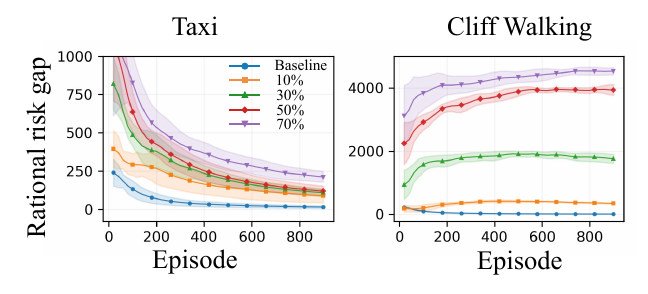}
  \caption{Special case of rational risk gap of DQN across different environment levels in Taxi-v3 and Cliff Walking environments. We evaluate DQN under increasing challenge levels of training environments (0\%, 10\%, 30\%, 50\%, 70\%), presenting the probability of action randomisation during training. }
  \label{fig:six_figures_level}
\end{figure}

\end{document}